%% file: log_2025.tex
\documentclass{article}
\usepackage{log_2025}						

\input{math_commands.tex}

\usepackage{graphicx}						
\usepackage{duckuments}						

\usepackage[utf8]{inputenc} 
\usepackage[T1]{fontenc}    
\usepackage{url}            
\usepackage{booktabs}       
\usepackage{amsfonts}       
\usepackage{nicefrac}       
\usepackage{microtype}      
\usepackage{xcolor}         

\usepackage{amsmath}
\usepackage{amssymb}
\usepackage{mathtools}
\usepackage{amsthm}

\usepackage{subfigure}

\theoremstyle{plain}
\newtheorem{theorem}{Theorem}[section]

\newtheorem{corollary}[theorem]{Corollary}
\theoremstyle{definition}
\newtheorem{definition}[theorem]{Definition}

\theoremstyle{remark}

\usepackage{bbding}
\usepackage{algorithm}
\usepackage{algpseudocode}
\usepackage{multirow}

\usepackage{xargs}                      
\usepackage[colorinlistoftodos,prependcaption,textsize=tiny]{todonotes}
\newcommandx{\unsure}[2][1=]{\todo[linecolor=red,backgroundcolor=red!25,bordercolor=red,#1]{#2}}
\newcommandx{\change}[2][1=]{\todo[linecolor=blue,backgroundcolor=blue!25,bordercolor=blue,#1]{#2}}
\newcommandx{\info}[2][1=]{\todo[linecolor=OliveGreen,backgroundcolor=OliveGreen!25,bordercolor=OliveGreen,#1]{#2}}
\newcommandx{\improvement}[2][1=]{\todo[linecolor=Plum,backgroundcolor=Plum!25,bordercolor=Plum,#1]{#2}}
\newcommandx{\thiswillnotshow}[2][1=]{\todo[disable,#1]{#2}}

\definecolor{darkred}{rgb}{0.9, 0, 0.1}

\usepackage[numbers,compress,sort]{natbib}	

\title[Efficient Neural Common Neighbor for Temporal Graph Link Prediction]{Efficient Neural Common Neighbor for Temporal Graph Link Prediction}

\author[X. Zhang et al.]{%
    Xiaohui Zhang\thanks{Equal Contribution} \\
    Institute for Artificial Intelligence, \\ Peking University \\
    \texttt{huihuang@stu.pku.edu.cn} 
    \And
    Yanbo Wang\footnotemark[1] \\
    Institute for Artificial Intelligence, \\ Peking University \\
    \texttt{wangyanbo@stu.pku.edu.cn} \\
    \And   
    Xiyuan Wang \\
    Institute for Artificial Intelligence, \\
    Peking University \\
    \texttt{wangxiyuan@pku.edu.cn} \\
    \And
    Muhan Zhang\thanks{Correspondence to: Muhan Zhang <muhan@pku.edu.cn>} \\
    Institute for Artificial Intelligence, Peking University \\
    State Key Laboratory of General Artificial Intelligence, Peking University \\
    \texttt{muhan@pku.edu.cn} \\
}

\begin{document}

\maketitle

\begin{abstract}
Temporal graphs are widespread in real-world applications such as social networks, as well as trade and transportation networks. Predicting dynamic links within these evolving graphs is a key problem. Many memory-based methods use temporal interaction histories to generate node embeddings, which are then combined to predict links. However, these approaches primarily focus on individual node representations, often overlooking the inherently pairwise nature of link prediction. While some recent methods attempt to capture pairwise features, they tend to be limited by high computational complexity arising from repeated embedding calculations, making them unsuitable for large-scale datasets like the Temporal Graph Benchmark (TGB).
To address the critical need for models that combine strong expressive power with high computational efficiency for link prediction on large temporal graphs, we propose Temporal Neural Common Neighbor (TNCN). Our model achieves this balance by adapting the powerful pairwise modeling principles of Neural Common Neighbor (NCN) to an efficient temporal architecture. TNCN improves upon NCN by efficiently preserving and updating temporal neighbor dictionaries for each node and by using multi-hop common neighbors to learn more expressive pairwise representations. 
TNCN achieves new state-of-the-art performance on Review from five large-scale real-world TGB datasets, 6 out of 7 datasets in the transductive setting and 3 out of 7 in the inductive setting on small- to medium-scale datasets. 
Additionally, TNCN demonstrates excellent scalability, outperforming prominent GNN baselines by up to 30.3 times in speed on large datasets. Our code is available at \href{https://github.com/GraphPKU/TNCN}{https:
//github.com/GraphPKU/TNCN}.
\end{abstract}

\section{Introduction}
Temporal graphs are increasingly employed in contemporary real-world applications like social and transaction networks, which evolve dynamically and exhibit distinct characteristics over time. Concurrently, Graph Neural Networks (GNNs)~\citep{scarselli2008graph} have emerged as prominent tools for graph representation learning, typically learning node embeddings by iteratively aggregating information from neighbors and demonstrating strong performance on various tasks. However, the temporal dimension introduces significant challenges; the discrete or continuous timestamps associated with graph edges define an evolutionary process and impose causality constraints, rendering many static GNN methodologies not directly applicable. This has led to the development of specialized temporal GNNs. Consequently, specialized temporal GNNs have been developed, often focusing on achieving both effective representation learning and computational efficiency for these large, dynamic structures. For instance, memory-based approaches like Temporal Graph Networks (TGN)~\citep{kumar2019predicting, trivedi2019dyrep, rossi2020temporal} are designed to efficiently learn short- and long-term dependencies, while Transformer-based models~\citep{wang2021tcl, xu2020inductive} utilize attention mechanisms to capture complex relationships.

\begin{figure*}[tp]
\vspace{-0.3cm}
    \centering
    \subfigure[a temporal graph]{
        \includegraphics[width=0.26\textwidth]{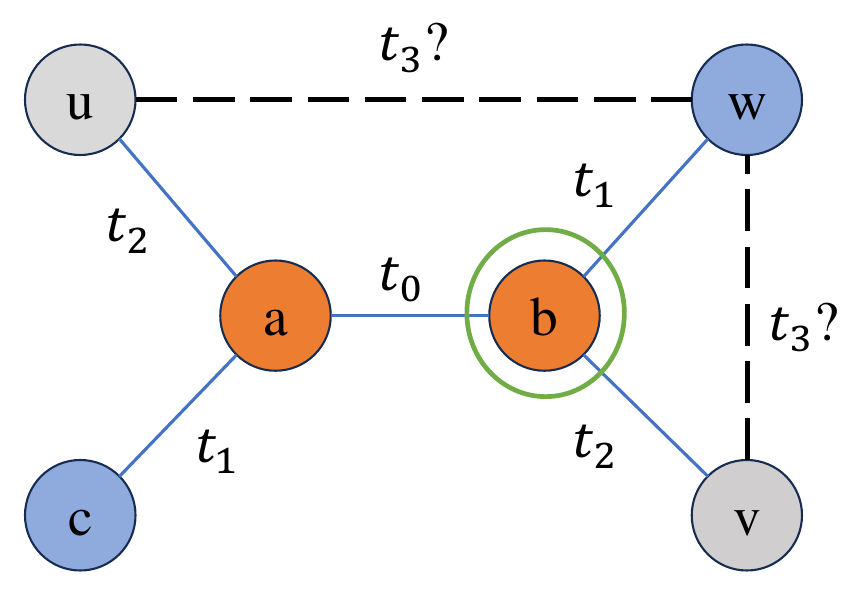}
    }
    \hspace{0.5cm}
    \subfigure[computation tree of $u$ and $v$]{
        \includegraphics[width=0.27\textwidth]{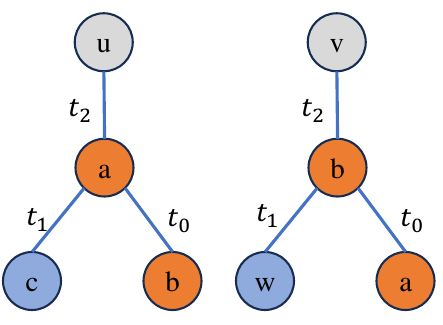}
    }
    \hspace{0.5cm}
    \subfigure[common neighbor method]{
        \includegraphics[width=0.32\textwidth]{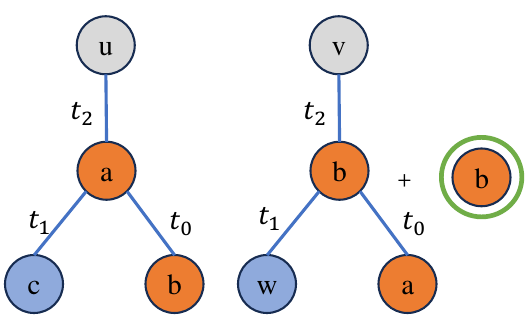}
    }
    \caption{Figure (a) shows a failure case of link prediction based on node-wise representation learning. Such methods cannot distinguish node $u$ and $v$ because they possess the same temporal computation tree in Figure (b), thus generating the same node representation. However, when we try to learn their pair-wise representation, i.e. $(u,w)$ and $(v,w)$, we can observe that $v$ has a temporal common neighbor $b$ with node $w$ while $u$ doesn't, as shown in Figure (c). Thus with the same computation graph, we only need to utilize the extra node $b$'s embedding to distinguish $(u,w)$ and $(v,w)$.}
    \label{fig:node-wise-repr}
    
    \vspace{-0.8cm}
\end{figure*}

Despite their advancements, many specialized temporal GNNs primarily generate node-wise representations. While effective for certain tasks, this focus can be insufficient for link prediction, where accurately capturing the relationship between pairs of nodes is critical. Node-wise methods may fail to distinguish between nodes that have similar individual features or local neighborhoods but different propensities to connect with a target node (as illustrated in Figure~\ref{fig:node-wise-repr}), thereby overlooking crucial relational patterns. Acknowledging such limitations, approaches from static graph link prediction, like the labeling trick~\citep{zhang2018link, zhang2021labeling} that emphasizes pair-wise representations, have shown considerable success. Consequently, researchers have extended these pair-wise learning concepts to temporal graphs, often by leveraging information from the evolving local neighborhoods of node pairs~\citep{wang2021inductive, luo2022neighborhood, yu2023towards}. However, these more expressive graph-based temporal models frequently incur substantial computational and memory costs, arising from the need to extract and process temporal neighborhood information for each prediction, which can hinder their application to large-scale scenarios.


This trade-off between computationally demanding expressive models and more efficient but potentially less informative ones highlights a critical need. Addressing this, we propose the \textbf{Temporal Neural Common Neighbor (TNCN)} model, which is designed to achieve both high efficiency and strong expressive power for temporal link prediction. TNCN builds upon a memory-based backbone, ensuring operational efficiency comparable to sequential update models. Crucially, it incorporates a Neural Common Neighbor component~\citep{wang2023neural}. This component is augmented with advanced operational techniques and extended for multi-hop common neighbor consideration, allowing TNCN to effectively model sophisticated link heuristics and learn detailed pairwise representations while retaining the efficiency of its memory-based foundation. As a result, TNCN is well-suited for large-scale temporal graph link prediction.


We conducted experiments on five large-scale real-world temporal graph datasets from TGB, where TNCN achieved new SOTA results on Review. Furthermore, evaluations on traditional small- to medium-scale datasets revealed TNCN achieving SOTA performance on 6 out of 7 datasets in the transductive setting and 3 out of 7 in the inductive setting,
demonstrating its effectiveness. To assess its scalability, datasets were selected with temporal edge counts ranging from  $\mathcal{O}(10^{5})$ to $\mathcal{O}(10^{7})$ and node counts from thousands to millions. On these large-scale datasets, TNCN achieved training speedups of 2.5x$\sim$5.9x and inference speedups of 1.8x$\sim$30.3x compared to graph-based models, while its time consumption remained comparable to that of memory-based models. 

\section{Preliminaries}
\begin{definition}
    \textbf{(Temporal Graph)} We mainly focus on the continuous time dynamic graph (\textbf{CTDG}). A CTDG can be typically represented as a sequence of interaction events: $\mathcal{G} = \{(u_1, v_1, t_1), \cdots, \\ (u_n, v_n, t_n)\}$, where $u,v$ stand for source and destination nodes and $\{t_i\}$ are chronologically non-decreasing timestamps. 
    Note that each node or edge can be attributed, that is, there may be node feature $x_u$ for $u$ or edge feature $e_{u,v}^t$ attached to the event $(u,v,t)$.
\end{definition}
\begin{definition}
    \textbf{(Problem Formulation)} Given the events before time $t^*$, i.e. $\{(u,v,t)\ |\ \forall\ t < t^*\}$, a link prediction task is to predict whether two specified node $u^*$ and $v^*$ are connected at time $t^*$.
\end{definition}
\begin{definition}
\label{def:CN}
    \textbf{(Temporal Neighborhood)} Given the center node $u$, the $k$-hop temporal neighbor set $(k\ge 0)$ before time $t$ is defined as $N_{k}^{t}(u)$. A node $v$ is in $N_{k}^{t}(u)$ if there exists a $k$-length path between $u$ and $v$, i.e. $\exists (u, w_1, w_2, \cdots, w_{k-1}, v)$ where $w_i \neq w_j, \forall i \neq j$. We also define the \textbf{$(\mathbf{i},\mathbf{j})$-hop common neighbor set} as follows: $w$ is an $(i,j)$-hop temporal common neighbor of $u$ and $v$ at time $t$ if $w\in N_{i}^{t}(u)$ and $w\in N_{j}^{t}(v)$. For simplicity we will denote the set as $\text{CN}_{(i,j)}^{t}(u,v) =  N_{i}^{t}(u)\cap N_{j}^{t}(v)$.
    Note that for $i=0$ (or $j=0$ similarly), we define the $0$-hop temporal neighbor set as  $N_{0}^{t}(u) = \{u\}$, and the $(0,j)$-hop common neighbor of $u$ and $v$ as $\text{CN}_{(0,j)}^{t}(u,v) =  N_{0}^{t}(u)\cap N_{j}^{t}(v) = \{u\}\cap N_{j}^{t}(v)$. Finally, the $K$-hop temporal neighborhood of node $u$ at time $t$ is defined as: $\mathop{\cup}\limits_{k=0}^{K}N_{k}^{t}(u)$. 
    
    With $(i,j)$-hop neighborhood information, we can perceive the local structure to a large extent and distinguish the difference between multi-hop common neighbors more precisely. 
\end{definition}

\begin{definition}
\label{pre:memory}
\textbf{Memory-based Backbone.}
Memory-based backbone has been widely adopted by various methods like \citep{kumar2019predicting, rossi2020temporal} to tackle dynamic graph learning. Its core component is the memory module that stores the node memory representations up to a certain time $t$. When a new event occurs, the memory of the source and destination nodes is updated with the message produced by the event. The computation can generally be represented as follows:
\begin{equation}
    \begin{aligned}
    &msg_{src}^t(u,v) = \textit{msgfunc}_{src}(e^t_{u,v}), 
    &mem^t_u = \textit{upd}_{src}(mem^{t-}_u, msg_{src}^t(u,v));\\
    &msg_{dst}^t(u,v) = \textit{msgfunc}_{dst}(e^t_{u,v}),    
    &mem^t_v = \textit{upd}_{dst}(mem^{t-}_v, msg_{dst}^t(u,v)).
    \end{aligned}
\end{equation}
where $msg^t$ stands for the message of the event, $mem^{t-}_u$ for the embedding of node $u$ before time $t$. 
\end{definition}

\section{Methodology}
Now we introduce our \textbf{Temporal Neural Common Neighbor} model. TNCN comprises several key modules: the classic Memory Module, the Temporal NCN Module with efficient CN Extractor, and the NCN-based Prediction Head. Special attention will be given to the Temporal CN Extractor, which is designed to efficiently extract temporal neighboring structures and obtain multi-hop CN information. The pipeline is illustrated in Figure~\ref{fig:pipeline}. A pseudocode is attached in Appendix~\ref{app:pseudocode}.

\begin{figure*}

    \centering
    \includegraphics[width=\textwidth]{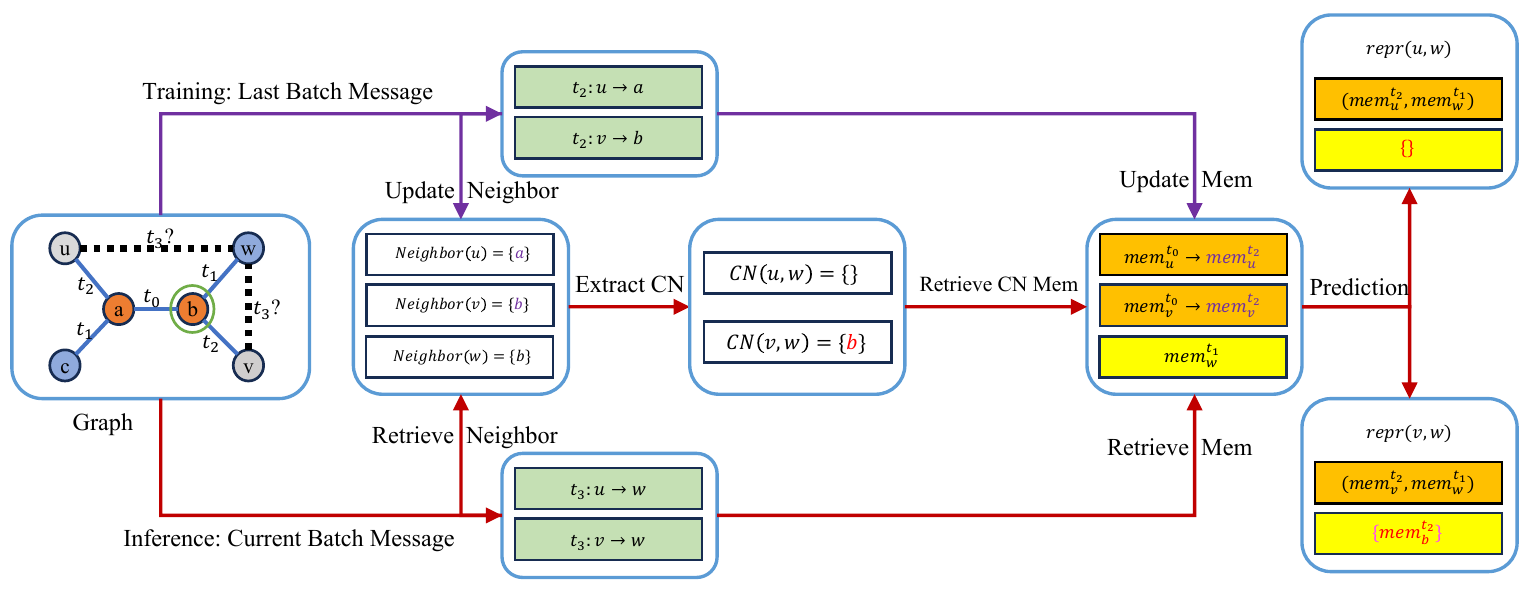}
    
    \caption{Pipeline of TNCN. TNCN operates through a sequential \textcolor{purple}{update} and \textcolor{darkred}{prediction} framework that processes successive batches of messages. During the update phase, TNCN updates the neighbor dictionary and the node memory. In the prediction phase, the model retrieves neighbors to identify common neighbors, leveraging the representations of the target nodes and their CNs for prediction.}
    \label{fig:pipeline}
    \vspace{-0.4cm}
\end{figure*}
\subsection{Memory Module}

Different from the static GNN used in traditional NCN, our model TNCN adopts a \textbf{memory-based backbone} from \citep{rossi2020temporal} to efficiently store and update the node memory, eliminating the need for repeated computation of node embeddings within successive temporal batches. 

Conforming to the standard pipeline in section \ref{pre:memory} of processing the node memory, for later link prediction or other downstream tasks, node embeddings can be obtained from their memory:
\begin{align}
emb^t_u = NN(mem^{t-}_u, \mathop{\cup}\limits_{v\in N^t_{1}(u)}mem^{t-}_v, \mathop{\cup}\limits_{t'<t} [e^{t'}_{u,v} \ ||\ \phi(t-t')]),
\end{align}
where $||$ stands for the concatenation operation.
Here NN has multiple choices, like Identity or simple static GNN~\citep{defferrard2016convolutional,velickovic2017graph,hamilton2017inductive}. In our implementation, we adopt Graph Transformer Convolution~\citep{shi2020masked}, which can pay more attention to the relation between different nodes. And we choose the time encoding $\phi$ presented in Time2Vec~\citep{kazemi2019time2vec} as TGN does.

\subsection{Temporal NCN Module with Efficient CN Extractor}
Our Temporal NCN Module can efficiently perform multi-hop common neighbor extraction with its CN extractor, and aggregate their neural embeddings to attain the node features. 

\textbf{Extended Common Neighbor.}
The definition of multi-hop common neighbors (CN) is given in Definition \ref{def:CN}, extending the traditional (1,1)-hop CN (i.e., nodes on 2-paths between $u$ and $v$) to arbitrary $(i,j)$-hop CN. Additionally, we define the zero-hop neighbor of a central node, i.e., $u$ is considered as a neighbor of itself, which will be utilized to calculate CNs with other nodes. Given source node $u$ and target node $v$, the (0,1)-hop and (1,0)-hop CN not only records the historical interactions between two nodes, but also reveals the frequency of their interactions. 

\textbf{Efficient CN Extractor.}
The CN Extractor is a crucial component of the TNCN model, contributing significantly to its high performance and scalability. It can efficiently gather pertinent information about a given center node and extract multi-hop common neighbors for a source-destination pair.

For each relevant node \( u \), the extractor stores its historical interactions with other nodes as both source and destination. After a batch of events is processed by the model, the storage is updated with the latest interactions. This allows us to maintain a record of all historical interactions up to a certain timestamp, effectively constructing a dynamic lookup dictionary for fast retrieval during subsequent inference. To strike a balance between memory consumption and model capacity, we save only the most recent $K$ events and relevant nodes for each center node, where $K$ is a hyperparameter determined by the specific dataset. 

To implement an efficient batch CN extractor, we organize the historical interactions in a \textit{Sparse Tensor}, representing the temporal adjacency matrix. Then we perform \textbf{self-multiplication} to generate high-order adjacency connectivity. Sparse matrix \textbf{hadamard product} is finally employed to obtain separate $(i,j)$-hop CNs. All these operations can be efficiently implemented by sparse tensor operators and are supported by GPU to facilitate fast, batch processing. Now we show the detailed procedure \textbf{as follows}. For some special and higher-order cases analysis, please refer to Appendix~\ref{app:cn}.

\textbf{Details of Common Neighbor Extraction.}
Our temporal CN extractor begins with a sparse matrix $A$ constructed from the interactions of related nodes. We then include three stages to precisely generate arbitrary $(i,j)$-hop CNs:

1. Generate up to $k$-hop neighbors. The original matrix $A$ only includes 1-hop neighbors. To extend this, we: (a) Use self-loops for 0-hop neighbors, denoted as $A^0$. (b) Perform sparse matrix multiplication (i.e., $A^k$) to include arbitrary k-hop neighbors.
Combining them, we obtain an updated neighborhood matrix set $\hat{A} = \{A^i\}_{i=0}^{K}$.

2. Extract neighbors for each source and destination node with corresponding indices in the same batch. Assume that we require the $k$-th hop neighbors of node $u$, then vector $A^k[id(u)]$ is the result, where $id(u)$ stands for the reindexed id for node $u$. $A^k[id(u)][id(v)] = w > 0$ if $v$ is a $k$-hop neighbor of $u$, otherwise this element is $0$. $w$ represents the historical interaction frequency. 

3. Obtain arbitrary $(i,j)$-hop CNs. We can perform hadamard product of $A^i[id(u)]$ and $A^j[id(v)]$ to acquire different hops of CNs. 
The operator can extract corresponding CNs for source-destination node pairs in a batch parallelly. 

    
    

By re-indexing the node IDs when generating $\hat{A}$ to prevent conflicts, the CN extractor can conduct the sparse matrix calculation, which are all performed in a Torch style that supports \textbf{batch operations}, thus enhancing parallelism and efficiency. 


The utilization of \textbf{Multi-hop Common Neighbors} significantly boosts TNCN's performance, resulting in higher scores in temporal link prediction tasks. Furthermore, by employing sparse tensors, our model achieves substantial reductions in both storage requirements and computational complexity, thereby decreasing time consumption and enhancing efficiency. 
We also give a comparison between our TNCN and traditional NCN in Table~\ref{tab:cmp} in Appendix~\ref{app:ncn}.

\subsection{NCN-based Prediction Head}
We finally construct our NCN-based representation as follows. For source and destination nodes, we perform an element-wise product. For multi-hop CN nodes, we aggregate their embeddings in each hop with sum pooling.
\begin{equation}
    X^t_{u,v} = emb_u^t \otimes emb_v^t, \quad  NCN_{(i,j)}(u,v) = \mathop{\oplus}\limits_{w \in \text{CN}^t_{(i,j)}(u,v)} emb_w^t.
\end{equation}
These embeddings are then concatenated as the final pair-wise representation:
\begin{equation}
    \begin{aligned}
        repr(u,v) &= [X^t_{u,v}\ ||\ (\mathop{||}\limits_{i,j}^{K}) NCN_{(i,j)}(u,v)].
    \end{aligned}
\end{equation}
We have used $\otimes$, $\oplus$, and $||$ to denote element-wise product, element-wise summation, and concatenation of vectors, respectively. The pair-wise representation $repr(u,v)$ for nodes $u$ and $v$ will be fed to a projection head to output the final link prediction.

\section{Efficiency and Effectiveness of TNCN}

In this section, we explore the two principal benefits of TNCN: efficiency and effectiveness. These advantages are demonstrated through an analysis of two core components within the framework for temporal graph link prediction: graph representation learning and link prediction methods. 

Temporal graph representation learning aims to develop an embedding function, denoted as $Emb$, which learns an embedding for each node encoding its structural and feature information within the graph. Specifically, given a new event represented as $(u, v, t)$, the function $Emb$ leverages prior events to generate meaningful embeddings. 
We first categorize graph representation learning approaches into two types: memory-based and $k$-hop-subgraph-based, according to \textbf{their temporal scope of evolved events}.

\begin{definition}
\textbf{Memory-based approach.} Given a new event $(u,v,t)$, if $Emb$ conforms to the following form, the method is referred to as a memory-based approach, which opts to maintain a dynamic, incrementally updated embedding for each node.
\begin{equation}
    Emb(u, t) = f_{\textit{emb}}(Mem(u, t')),~~~~
    Emb(v, t) = f_{\textit{emb}}(Mem(v, t')),
\end{equation}
where the $Mem$ can be obtained from the pipeline in Definition \ref{pre:memory}, and $f_{\textit{emb}}$ is a learnable function.
\end{definition}

\begin{definition}
\textbf{$k$-hop-subgraph-based approach.} Given a new event $(u,v,t)$, if $Emb$ conforms to the following form, the method is defined as a subgraph-based approach, which chooses to recalculate node embeddings by considering the entire historical events.
\begin{equation}
    Emb(u, t) = f_{\textit{emb}}(\mathcal{G}_{u, <t}^{k}), ~~~~ Emb(v, t) = f_{\textit{emb}}(\mathcal{G}_{v, <t}^{k}),
\end{equation}
where $\mathcal{G}_{u, <t}^{k}$ is a subgraph induced from $\mathcal{G}$ by node $u$'s $k$-hop temporal neighborhood $\mathop{\cup}\limits_{k=0}^{K}N_{k}^{t}(u)$, containing only the edges (events) with time $t' < t$, and $f_{\textit{emb}}$ is a learnable function.
\end{definition}

\subsection{Effectiveness}
The analysis begins by assessing the effectiveness of the two paradigms. 
To do so, we first introduce the concept of $k$-hop event~\citep{LOVASZ1993RandomWO}. 

\begin{definition}
\textbf{$k$-hop event \& monotone $k$-hop event.} 
A \textit{$k$-hop event} is a sequence of consecutive edges $\{ (u_i, u_{i+1}, t_{u_i, u_{i+1}}) \mid i \in \{0, \ldots, k-1\}, k \geq 1 \}$ connecting the initial node $u_0$ to the final node $u_k$. For example, $\{(u,x,t'),(x,v,t)\}$ is a $2$-hop event. In the case where $k=1$, the $k$-hop event reduces to a single interaction $(u, v, t)$. A \textit{monotone $k$-hop event} is a $k$-hop event in which the sequence of timestamps $\{ t_{u_i, u_{i+1}} \mid i \in \{0, \ldots, k-1\}, k \geq 1\}$ is strictly monotonically increasing.
\end{definition}
Then, we analyze the expressiveness of the two approaches in terms of encoding $k$-hop event.
\begin{theorem}
\label{the:encoding}
(Ability to encode $k$-hop events). Given a $k$-hop event $\{(u_i, u_{i+1}, t_{u_i, u_{i+1}}) \mid i \in \{0,\ldots,k-1\}, k\geq1\}$, if the node embedding of $u_0$ at time $t_{u_0, u_1}$ can be reversely recovered from the encoding $\textit{Enc}(\{(u_i, u_{i+1}, t_{u_i, u_{i+1}}) \mid i \in \{0,\ldots,k-1\}, k\geq1\})$, then we say the encoding function \textit{Enc} is capable of encoding the $k$-hop event. The following results outline the encoding capabilities of different learning paradigms:
\begin{itemize}
    \item Memory-based approaches can encode any $k$-hop events with $k=1$.
    \item Memory-based approaches can encode any monotone $k$-hop events with arbitrary $k$.
    \item $k$-hop-subgraph-based approaches can encode any $k'$-hop events with $k' \leq k$.
\end{itemize}
\end{theorem}

The proof is in Appendix~\ref{app:proof}. From Theorem \ref{the:encoding}, we can conclude that 1) memory-based approaches have superior expressiveness in encoding $k$-hop events compared to 1-hop-subgraph-based approaches, and 2) memory-based approaches have superior expressiveness in encoding monotone $k$-hop events than $k'$-hop-subgraph-based approaches when $k'<k$.

While the memory-based approach does not consistently rival the expressiveness of the $k$-hop-subgraph paradigm, it possesses advantages in monotone events and long-history scenarios (where $k$-hop subgraphs would be unaffordable to extract).

\begin{corollary}
    If we use up to $k$ hop neighborhood information of central node $u$, then TNCN can capture at least $(k+1)$-hop subgraph information around $u$. 
\end{corollary}
This is because TNCN with memory-based backbone can obtain additional 1-hop information regardless of the time monotony, i.e. arbitrary central node can interact with its neighbor when the edge between them exists. This 
can extend TNCN's capability for free.

We also demonstrate the effectiveness of TNCN in the following theorem~\ref{the:cn} with the proof in Appendix~\ref{app:proof}. The first part shows that it can capture three important pairwise features commonly used as effective link prediction heuristics, namely Common Neighbors (CN), Resource Allocation (RA), and Adamic-Adar (AA)~\citep{newman2001clustering, adamic2003friends, zhou2009predicting}. The experimental results in Appendix~\ref{app:exps} further validate this claim. In the second part we reveal that TNCN is strictly more expressive than some traditional temporal graph networks including Jodie~\citep{kumar2019predicting}, DyRep~\citep{trivedi2019dyrep}, TGN~\citep{rossi2020temporal} and TGAT~\citep{xu2020inductive} under the same condition, which are widely used as baselines. 
\begin{theorem}
\label{the:cn}
(Expressivity of TNCN)
\begin{enumerate}
    \item TNCN is strictly more expressive than CN, RA, and AA.
    \item TNCN is strictly more expressive than Jodie with the same dimension of time encoding, DyRep with the same aggregation function, TGAT with the same attention layers and neighbors, and TGN under identical condition for all module choices.
\end{enumerate}
\end{theorem}

From this theorem we can find that TNCN extends the previous generic memory-based framework of temporal graph networks with explicitly adopting the neural embeddings of common neighbors. This method can serve as a complementary addition for learning pair-wise representations.


\subsection{Efficiency}
We then turn our attention to the efficiency of the two approaches. A pivotal factor is the frequency with which individual events are incorporated into computations. In memory-based approaches, each event is utilized \textbf{a single time} for learning, immediately following its associated prediction. Conversely, in the $k$-hop-subgraph-based method, an event may be employed \textbf{multiple times}, as it is revisited in different nodes' temporal neighborhood and repeated been processed within each subgraph's encoding (such as message passing) process. This discrepancy leads to divergent cumulative frequencies of event utilization throughout the learning process, resulting in the huge efficiency advantage of memory-based methods. We formalize this observation as:

\begin{theorem}
\label{the:time}
    (Learning method time complexity). Denote the time complexity of a learning method as a function of the total number of events processed during training. For a given graph $\mathcal{G}$ with the number of nodes designated as $|\mathcal{N}|$ and the number of edges as $|\mathcal{E}|$, the following assertions hold:
    \begin{itemize}
        \item For memory-based approaches, the time complexity is $\Theta\left(|\mathcal{E}|\right)$.
        \item For $k$-hop-subgraph-based approaches with $k=1$, the lower-bound time complexity is $\Omega\left(\frac{|\mathcal{E}|^2}{|\mathcal{N}|}\right)$, and the upper-bound time complexity is $\mathcal{O}\left(\frac{|\mathcal{E}|^2}{|\mathcal{N}|} + |\mathcal{E}||\mathcal{N}|\right)$.
        \item For $k$-hop-subgraph-based approaches with $k=2$, the upper-bound time complexity is $\mathcal{O}\left(\left(\frac{|\mathcal{E}|^2}{|\mathcal{N}|} + |\mathcal{E}||\mathcal{N}|\right)^{\frac{3}{2}}\right)$.
    \end{itemize}
\end{theorem}

The proof is attached in Appendix~\ref{app:proof} with part of the proof based on a classic conclusion from \citet{DECAEN1998245} in the graph theory.
Following Theorem \ref{the:time}, it becomes evident that the computational overhead incurred by a memory-based method is significantly lower than that of a subgraph-based method, particularly as $k$ increases. These results highlight the advantages of memory-based methods in mitigating the computational efficiency challenges associated with large-scale temporal graphs.

So based on the memory backbone, our TNCN furthur utilizes an efficient CN extractor, eliminating the necessity for message passing on entire graphs. Consequently it achieves a unified optimization objective: \textbf{to avoid message passing on entire subgraphs in favor of non-repetitive operations}.

To summarize, TNCN introduces the extended common neighbor approach with the efficient CN extractor. This method serves as a complementarity for learning pair-wise representations, while  culminating in a cohesive solution that is both efficient and effective.

\section{Related Work}
\noindent\textbf{Memory-based Temporal Graph Representation Learning.} 
Memory-based models learn node memory using continuous events with non-decreasing timestamps. Researchers have proposed several memory-based methods including JODIE \citep{kumar2019predicting},  DyRep \citep{trivedi2019dyrep} and TGN \citep{rossi2020temporal}. These methods are superior in higher efficiency, while lacking in capturing structural information.

\noindent\textbf{Graph-based Temporal Graph Representation Learning.}
Subsequent works have incorporated the neighborhood structure into temporal graph learning. CNE-N \citep{cheng2024co} employs a hash table to map interactions and computes co-neighbor encodings to predict future links. Unlike our TNCN which directly utilizes CN embeddings in prediction, CNE-N simply counts neighbor numbers.
Additionally, while CNE-N manages recent interactions through hash tables, TNCN adopts a monotonic storage scheme. NAT \citep{luo2022neighborhood} similarly builds a multi-hop node dictionary to compress neighbors, but it may suffer from hash collisions. DyGFormer \citep{yu2023towards} encodes one-hop neighbors and their co-occurrences, then uses a Transformer for predictions, requiring repeated neighborhood sampling and computations. It only models 1-hop neighbor frequency and does not use CN embeddings. In contrast, TNCN supports multi-hop neighbors and includes their embeddings, enabling richer representations. 

We also include detailed introduction to more related works in Appendix \ref{app:related-work}.

\section{Experiments}
This section assesses TNCN's effectiveness and efficiency by answering the following questions:

\textbf{Q1:} What is the performance of TNCN compared with state-of-the-art baselines? \\
\textbf{Q2:} What is the computational efficiency of TNCN in terms of time consumption? \\
\textbf{Q3:} Do the extended common neighbors bring benefits to original common neighbors?

\subsection{Experimental Settings}

\textbf{Datasets.}
We evaluate our model on five large-scale real-world datasets for temporal link prediction from the \textbf{\textit{Temporal Graph Benchmark}}~\citep[]{huang2023temporal}. These datasets span several distinct fields: co-editing network on Wikipedia, Amazon product review network, cryptocurrency transactions, directed reply network of Reddit, and crowdsourced international flight network. They vary in scales and time spans. Additional details about the datasets are provided in Appendix \ref{app:datasets}. We set the evaluation metric as \textbf{Mean Reciprocal Rank (MRR)} consistent with the TGB official leaderboard.

\textbf{Baselines.}
We systematically evaluate our proposed model TNCN against a diverse set of baselines: a heuristic algorithm Edgebank \citep{yu2023towards}, memory-based models Jodie \citep{kumar2019predicting}, DyRep \citep{trivedi2019dyrep} and TGN \citep{rossi2020temporal} that obviate the need for frequent temporal subgraph sampling, and GraphMixer \citep{cong2023we} which employs an MLP-mixer. We also include various graph-based models such as CAWN \citep{wang2021inductive}, TGAT \citep{xu2020inductive}, TCL \citep{wang2021tcl}, NAT \citep{luo2022neighborhood}, DyGFormer \citep{yu2023towards}, CNE-N \citep{cheng2024co} and TPNet \citep{lu2024improvingtemporallinkprediction}, which learn from neighborhood structure information. 

Here we evaluate our TNCN under two similar but different settings, the official setting (``official'') and the new setting (``ns''). ``*-official'' strictly complies to the \textbf{official setting} of TGB evaluation policy, using both \textit{streaming setting} and \textit{lag-one scheme} for both memory update and neighborhood awareness. Streaming setting means the information of the validation and test sets can only be employed for updating the memory without any back propagation. Lag-one scheme implies that the model can access only the information from \textbf{before the current batch} for predictions; in other words, the latest usable batch is the previous one. This applies to not only the memory, but also the \textbf{neighborhood awareness}. ``*-ns'' obeys the streaming setting but considers the interactions within the same batch before the current prediction time. This allows the model to utilize more recent neighborhood information, potentially giving it unfair advantages in datasets where recent interactions are crucial. Methods ``*-official'' use the former setting while others report the latter. For a fair evaluation and comparison, here we display the performance of our TNCN under both settings. 


\begin{table*}[t]

\begin{center}
    \caption{Test Performance of different models under MRR metric. The top three are emphasized by \textcolor{darkred}{red}, \textcolor{blue}{blue} and \textbf{bold} fonts. ‘-’ denotes scenarios where a specific method was either not applied to the dataset or was unable to complete the validation and testing phases within a reasonable timeframe. 
    }
    \label{tab:main}
    
    \resizebox{\textwidth}{!}{
    \begin{tabular}{lccccc}
        \toprule

Model & Wiki        & Review      & Coin        & Comment       & Flight     \\ 
\midrule
JODIE & 0.631 $\pm$ 1.69 & \textbf{0.414 $\pm$ 0.15} & - & - & - \\
DyRep & 0.519 $\pm$ 1.95 & 0.401 $\pm$ 0.59 & 0.452 $\pm$ 4.60 & 0.289 $\pm$ 3.30 & 0.556 $\pm$ 1.40 \\
TGAT & 0.599 $\pm$ 1.63 & 0.196 $\pm$ 0.23 & 0.609 $\pm$ 0.57 & 0.562 $\pm$ 2.11 & - \\
TGN-official  & 0.528 $\pm$ 0.06 &     0.387 $\pm$ 0.02 & 0.737 $\pm$ 0.03 & 0.622 $\pm$ 0.02   & 0.705 $\pm$ 0.02    \\ 
TGN-ns & 0.689 $\pm$ 0.53 & 0.375 $\pm$ 0.23 & 0.586 $\pm$ 3.70 & 0.379 $\pm$ 2.10 & - \\
CAWN & 0.730 $\pm$ 0.60 & 0.193 $\pm$ 0.10 & - & - & - \\
EdgeBank(tw) & 0.633 & 0.029 & 0.574  & 0.149  & 0.387 \\
EdgeBank(un) & 0.525 & 0.023 & 0.359 & 0.129 & 0.167 \\
TCL & 0.781 $\pm$ 0.20 & 0.165 $\pm$ 1.85 & 0.687 $\pm$ 0.30 & 0.701 $\pm$ 0.83 & - \\
GraphMixer & 0.598 $\pm$ 0.39 & 0.369 $\pm$ 1.50 & 0.756 $\pm$ 0.27 & \textbf{0.762 $\pm$ 0.17} & - \\
NAT & 0.749 $\pm$ 1.00 & 0.341 $\pm$ 2.00 & - & - & - \\
DyGFormer & 0.798 $\pm$ 0.42 & 0.224 $\pm$ 1.52 & 0.752 $\pm$ 0.38 & 0.670 $\pm$ 0.14 & - \\
CNE-N & \textbf{0.802 $\pm$ 0.20} & 0.261 $\pm$ 0.25 & \textcolor{blue}{0.772 $\pm$ 0.21} & \textcolor{blue}{0.790 $\pm$ 0.14} & - \\
TPNet & \textcolor{darkred}{0.827 $\pm$ 0.01} & - & \textcolor{darkred}{0.832 $\pm$ 0.01} & \textcolor{darkred}{0.825 $\pm$ 0.06} & \textcolor{darkred}{0.884 $\pm$ 0.01} \\
\midrule
TNCN-official & 0.724 $\pm$ 0.01 & \textcolor{blue}{0.419 $\pm$ 0.09} & 0.770 $\pm$ 0.06 & 0.727 $\pm$ 0.12 & \textbf{0.817 $\pm$ 0.04}\\
TNCN-ns & \textcolor{blue}{0.803 $\pm$ 0.01} & \textcolor{darkred}{0.427 $\pm$ 0.06} & \textbf{0.771 $\pm$ 0.04} & 0.705 $\pm$ 0.12 & \textcolor{blue}{0.831 $\pm$ 0.03}\\

\bottomrule
    \end{tabular}
    }
    
\end{center}
\vspace{-0.5cm}

\end{table*}

\subsection{Experimental Results}

\noindent\textbf{Reply to Q1: TNCN possesses remarkable performance.} 
We conducted comprehensive evaluations of prevailing methods on TGB. The main results are summarized in Table \ref{tab:main}. It is evident from the table that TNCN attains \textbf{new SOTA performance on Review dataset}. Additionally, TNCN demonstrates competitive results on the remaining datasets, ranking 2nd or/and 3rd on Wiki, Coin and Flight. TNCN almost consistently surpasses both memory-based models such as TGN and DyRep, and graph-based methods like DyGFormer and NAT.
The dataset where TNCN still exhibits a large gap from the top-three baselines is Comment. This may be ascribed to the high ``surprise index'' of this non-bipartite dataset, wherein prior events have a diminished correlation with subsequent events, potentially reducing the impact of CNs. (More details about \textit{Surprise index} are in Appendix~\ref{app:datasets}.)

To obtain a more comprehensive evaluation, we have also conducted experiments on the previous small and medium datasets under both transductive and inductive settings. In these experiments, TNCN achieves \textbf{6 new SOTA} out of 7 datasets \textbf{in transductive setting} and \textbf{3 in inductive setting}, further facilitating its strong performance. The overall performance and some additional results (such as TGN with heuristics, \textit{etc.}) can be referred to Appendix
 \ref{app:exps}.

\noindent\textbf{Reply to Q2: TNCN shows great scalability on large datasets.}
To evaluate computational efficiency, we collected the time consumption 
on Wiki and Review datasets, as depicted in Figure \ref{fig:time}. Compared with memory-based methods, TNCN exhibits a comparable order of magnitude in terms of time consumption. However, when benchmarked against graph-based models, TNCN demonstrates a substantial acceleration, achieving approximately 2.5 to 5.9 times speedup during the training phase and a 1.8 to 30.3 times increase in inference speed. In the Table ~\ref{tab:efficiency-4090} we provide a full comparison of the time consumption for different models over the five TGB datasets.  Notably, the scalability concerns become even more evident as the size of the dataset expands; several graph-based models cannot complete the validation and testing processes within a reasonable time budget. The primary factors contributing to TNCN’s efficiency are the synergistic, time-efficient design of its two core components and the implementation of the Efficient CN Extractor that facilitates batch operations through parallel processing. For detailed statistics about TNCN and NAT, please refer to Appendix~\ref{app:time}.

\begin{figure}
    \vspace{-0.1cm}
    \centering
    \subfigure[tgbl-wiki (log scale)]{
        \includegraphics[width=0.7\linewidth]{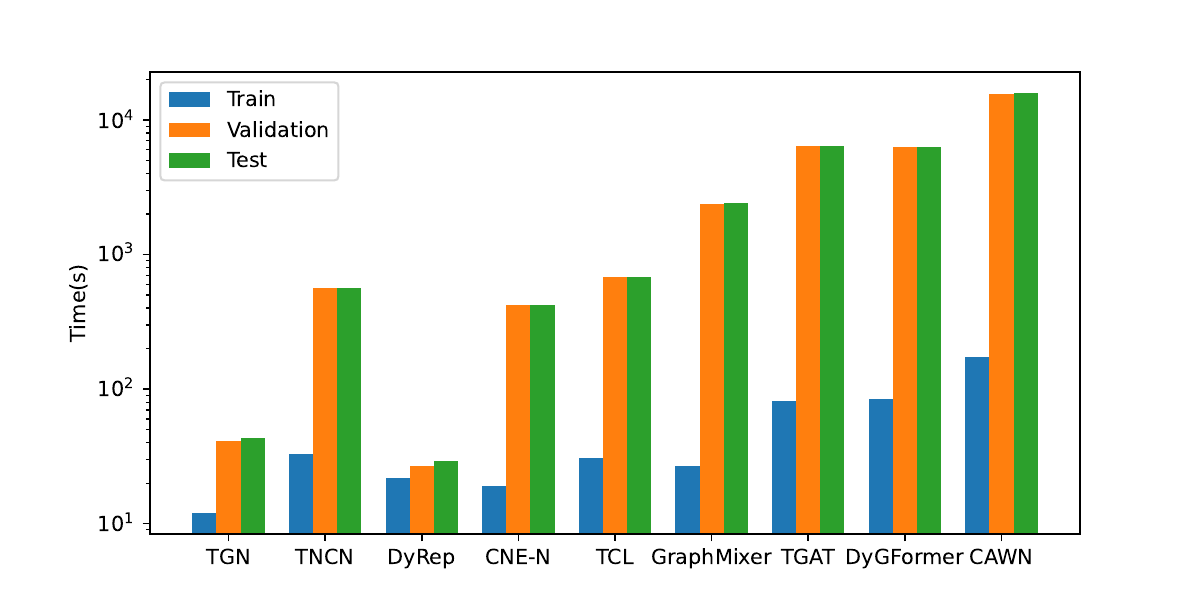}
    }
    \subfigure[tgbl-review (linear scale)]{
        \includegraphics[width=0.7\linewidth]{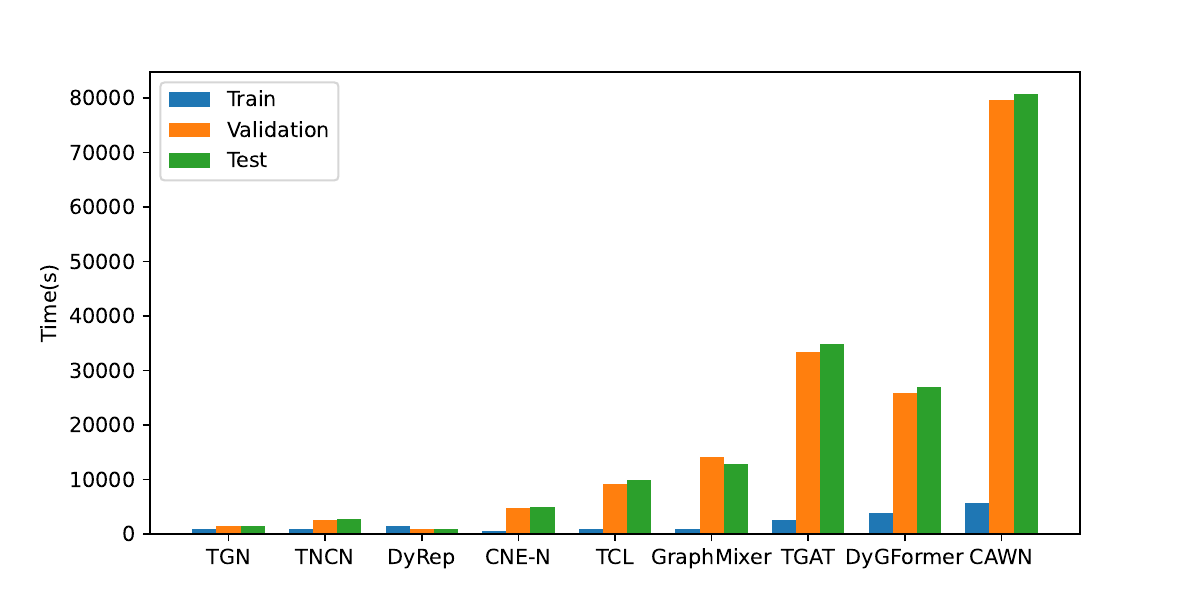}
    }
    \caption{Time Consumption of Memory and Graph-based Method on Wiki and Review Datasets.}
    \label{fig:time}
\end{figure}

\subsection{Ablation Study}
\noindent\textbf{Reply to Q3: Extended CN brings improvements.}
To elucidate the benefits of extended CN, we conducted an ablation study under \textbf{official setting} on the hop range of common neighbors. The results are shown in Table \ref{tab:ablation}. Here we use notation "$k$-hop CN" to simply denote the CNs up to $(k,k)$-hop. The conventional NCN method considers only (1,1)-hop CN. However, this approach may not be universally applicable across all temporal networks. For instance, \textit{bipartite graphs} lack such (1,1)-hop CN in their structure, necessitating the consideration of 2-hop CN. Additionally, memory-based methods may omit a notable aspect: they generally find it difficult to quantify the frequency of interactions between a given pair of nodes, which brings the need for 0-hop neighborhood.

To address these limitations, we have expanded the original (1,1)-hop CN to 0$\sim$k-hop CN. The results indicate that TNCN utilizing 0$\sim$1-hop CN markedly surpasses the (1,1)-hop CN on various datasets. This enhancement underscores the significance of the introduced 0-hop neighbors to our architecture. Nevertheless, the inclusion of 2-hop CN yields mixed results across datasets. We also show the result of TNCN without temporal NCN module (which is TGN), revealing the effectiveness of this module.

\begin{table}[t]

\begin{center}
    \caption{Test performance of TGN and TNCN with different ranges of common neighbors.}
    
    \label{tab:ablation}
    \resizebox{0.7\linewidth}{!}{
    \begin{tabular}{lcccc}
        \toprule
        Model & Wiki  & Review        & Coin  & Comment   \\ 
        \midrule
        TGN & 0.528  & 0.387            & 0.737 & 0.622    \\ 
        TNCN-1-hop-CN   & 0.621      & \textbf{0.419} & 0.737 & 0.641 \\ 
        TNCN-0$\sim$1-hop-CN  & 0.720        & 0.298       & 0.739 & \textbf{0.727}    \\ 
        TNCN-0$\sim$2-hop-CN & \textbf{0.724} & 0.317 & \textbf{0.770} & 0.662   \\ 
        \bottomrule
    \end{tabular}
    }
\end{center}
\vspace{-0.3cm}

\end{table}

We have also conducted experiments for parameter analysis. Please refer to Appendix \ref{app:params-app} for details.

\section{Conclusion and Limitation}
We propose TNCN for temporal graph link prediction, which employs a temporal common neighbor extractor combined with a memory-based node representation learning module. TNCN has achieved new state-of-the-art results on several real-world datasets while maintaining excellent scalability to handle large-scale temporal graphs.

However, based on our observation of TNCN's performance on the Comment dataset, there are some limitations in our model. Specifically, non-bipartite datasets with high surprise values, such as the Comment dataset, tend to make it more challenging for TNCN to accurately predict the probability of future connections. This indicates that while TNCN performs well overall, it may struggle with datasets that exhibit high variability or unexpected patterns. Further research is needed to address these challenges and improve the model's robustness in such scenarios.

\section*{Acknowledgement}
This work is supported by the National Key R$\&$D Program of China (2022ZD0160300) and National Natural Science Foundation of China (62276003).

\bibliographystyle{unsrtnat}
\bibliography{reference}
\newpage 
\appendix

\section{Datasets}
\label{app:datasets}
Table~\ref{tab:TGB} shows some detailed datasets statistics of TGB and \ref{tab:previous} shows several temporal graph datasets commonly used by previous work. Through the two tables we can observe that TGB official datasets possess temporal graphs with larger scale to 10 million, 10 times surpassing the largest previous datasets such as LastFM.
With the aim to examine our TNCN model's efficiency, we choose the increasingly accepted datasets TGB in the main table.

\begin{table*}[htp]
\begin{center}
    \caption{TGB Dataset Statistics.}
    \label{tab:TGB}
    
    \resizebox{\textwidth}{!}{
    \begin{tabular}{lcccccc}
        \toprule
    Dataset & Domain & Nodes & Edges & Steps & Surprise & Edge Properties \\ 
        \midrule
    tgbl-wiki    & interact & 9,227 & 157,474 & 152,757 & 0.108 & W: $\times$, Di: \checkmark, A: \checkmark \\
    tgbl-review  & rating   & 352,637 & 4,873,540 & 6,865 & 0.987 & W: \checkmark, Di: \checkmark, A: $\times$ \\
    tgbl-coin    & transact & 638,486 & 22,809,486 & 1,295,720 & 0.120 & W: \checkmark, Di: \checkmark, A: $\times$ \\
    tgbl-comment & social   & 994,790 & 44,314,507 & 30,998,030 & 0.823 & W: \checkmark, Di: \checkmark, A: \checkmark \\
    tgbl-flight  & traffic  & 18143 & 67,169,570 & 1,385 & 0.024 & W: $\times$, Di: \checkmark, A: \checkmark \\
        \bottomrule
    \end{tabular}
    }
    \end{center}

\end{table*}

Here ``Surprise index''~\citep{poursafaei2022towards} refers to the ratio of test edges that are not seen during training, which can be calculated as $\frac{|E_{test}/E_{train}|}{|E_{test}|}$. Low surprise index implies that memory-based methods such as Edgebank~\citep{poursafaei2022towards} may potentially achieve good performance, while high surprise may require more inductive capability. The surprise index varies across TGB datasets.

\begin{table*}[htp]
    \begin{center}
    \caption{Previous Dataset Statistics.}
    \label{tab:previous}

    \resizebox{\textwidth}{!}{
    \begin{tabular}{lcccccccc}
    \toprule
 Datasets & Domains & Nodes & Links & N\&L Feat 
 & Bipartite & Duration & Unique Steps & Time Granularity \\
 \midrule
Wikipedia & Social & 9,227 & 157,474 & – \& 172 & $\checkmark$ & 1 month & 152,757 & Unix timestamps \\
Reddit & Social & 10,984 & 672,447 & – \& 172 & $\checkmark$ & 1 month & 669,065 & Unix timestamps \\
MOOC & Interaction & 7,144 & 411,749 & – \& 4 & $\checkmark$ & 17 months & 345,600 & Unix timestamps \\
LastFM & Interaction & 1,980 & 1,293,103 & – \& – & $\checkmark$ & 1 month & 1,283,614 & Unix timestamps \\
Enron & Social & 184 & 125,235 & – \& – & $\times$ & 3 years & 22,632 & Unix timestamps\\
UCI & Social & 1,899 & 59,835 & – \& – & $\times$ & 196 days & 58,911 & Unix timestamps \\
\bottomrule
    \end{tabular}}
    \end{center}

\end{table*}

\section{Related Work}
\label{app:related-work}
\subsection{Memory-based Temporal Graph Representation Learning}
Temporal graph learning has garnered significant attention in recent years. A classic approach in this domain involves learning node memory using continuous events with non-decreasing timestamps. \citet{kumar2019predicting} propose a coupled recurrent neural network model named JODIE that learns the embedding trajectories of users and items. Another contemporary work DyRep~\citep{trivedi2019dyrep} aims to efficiently produce low-dimensional node embeddings to capture the communication and association in dynamic graphs. \citet{rossi2020temporal} introduces a memory-based temporal neural network known as TGN, which incorporates a memory module to store temporal node representations updated with messages generated from the given event stream. Apan~\citep{Wang_2021} advances the methodology by integrating asynchronous propagation techniques, markedly increasing the efficiency of handling large-scale graph queries. EDGE~\citep{chen2021efficient} emerges as a computational framework focusing on increasing the parallelizability by dividing some intermediate nodes in long streams each into two independent nodes while adding back their dependency by training loss. \citet{chen2023recurrent} extend the update method for the node memory module, introducing an additional hidden state to record previous changes in neighbors. Complementing these efforts, additional contributions such as Edgebank~\citep{poursafaei2022towards} and DistTGL~\citep{zhou2023disttgl} have been directed towards formalizing and accelerating memory-based temporal graph learning methods.

\subsection{Graph-based Temporal Graph Representation Learning}

Subsequent works have incorporated the temporal neighborhood structure into temporal graph learning. CAWN \citep{wang2021inductive} employs random anonymous walks to model the neighborhood structure. TCL \citep{wang2021tcl} samples a temporal dependency interaction graph that contains a sequence of temporally cascaded chronological interactions. TGAT \citep{xu2020inductive} considers the temporal neighborhood and feeds the features into a temporal graph attention layer utilizing a masked self-attention mechanism. NAT \citep{luo2022neighborhood} constructs a multi-hop neighboring node dictionary to extract joint neighborhood features and uses RNN to recursively update the central node's embedding. 
DyGFormer \citep{yu2023towards}, instead, leverages one-hop neighbor embeddings and the co-occurrence of neighbors to generate features, which are well-patched and subsequently fed into a Transformer \citep{vaswani2017attention} decoder to obtain the final prediction. FreeDyG~\citep{tian2023freedyg} also utilizes historical interaction frequency akin to DyGFormer, afterwards transforming it with Fast Fourier Transform (FFT) and IFFT through the frequency domain. 
LPFormer~\citep{Shomer_2024} attempts to adaptively learn the pairwise encodings via graph attention module, utilizing relative position, ppr value and neighboring information to obtain the score. CNE-N~\citep{cheng2024co} uses a hash table to map an interaction event to its position. It calculates the co-neighbor encoding for each (neighbor - end node) pair within the local subgraph, recording the number of their common neighbors. These information are then concatenated to predict the probability of the future link. TPNet \citep{lu2024improvingtemporallinkprediction} constructs temporal walk matrices via random propagation to simultaneously consider both temporal and structural
information. Pairwise and auxiliary feature are then decoded to make the prediction. Another work RepeatMixer \citep{zou2024repeatawareneighborsamplingdynamic} pays more attention to the repeat-aware neighbors. Such neighbor sequences are then leveraged and adaptively aggregated to learn the temporal patterns.
DyG-Mamba~\citep{li2024dygmambacontinuousstatespace} introduces SSM to dynamic graph learning. It first extracts the first-hop interaction sequence between the given node pairs and encodes them. The time-span encoding between any two continuous timestamps are also computed, serving as control signals for the continuous SSM to obtain the final representation.
CrossLink~\citep{crosslink10.1145/3696410.3714792} represents the graph evolution by a sequence of temporal events, where each event representation is obtained from a graph encoder. It finally utilizes a decoder-only transformer to model the sequence and predicts the next token, i.e. the link existence. 
EAGLE~\citep{li2025speedmeetsaccuracyefficient} aggregates the most recent neighbors of a central node and leverage temporal personalized PageRank value to capture the structural pattern, afterwards using adaptive weights to dynamically merge these features to get the prediction result.

\subsection{Link Prediction Methods}
Link prediction is a fundamental task in graph analysis, aiming to determine the likelihood of a connection between two nodes. Early investigations posited that nodes with greater similarity tend to be connected, which led to a series of heuristic algorithms such as Common Neighbors, Katz Index, and PageRank~\citep{newman2001clustering, katz1953new, page1999pagerank}. With the advent of GNNs, numerous methods have attempted to utilize vanilla GNNs for enhancing link prediction, revealing sub-optimal performance due to the inability to capture important pair-wise patterns such as common neighbors~\citep{zhang2018link, zhang2021labeling, liang2022can}. Subsequent research has focused on infusing various forms of inductive biases to retrieve intricate pair-wise relationships. For instance, SEAL~\citep{zhang2018link}, Neo-GNN~\citep{yun2021neo}, and NCN~\citep{wang2023neural} have integrated neighbor-overlapping information into their design. BUDDY~\citep{chamberlain2022graph} and NBFNet~\citep{zhu2021neural} have concentrated on extracting higher-order structural information. Additionally, \citet{mao2023revisiting, li2024evaluating} have contributed to a more unified framework encompassing different heuristics.

\section{TNCN Model Configuration}
\paragraph{Network Choice} In our experiment, the changeable neural networks are chosen as follows:

In Memory Module, we choose $Identity$ as $\textit{msgfunc}$ and GRU as $upd$. In inference stage we process node memory with Graph Attention Embedding to get the temporal representation. As for Prediction Head, we finally choose $MLP$ as the $repr$ function.

\paragraph{Hyper-parameter}
Several detailed hyper-parameters for TNCN are shown in Table~\ref{tab:hyper}, which can help researchers to reproduce the experiment performance as reported in this paper. 

\begin{table}[htbp]
    \centering
    \caption{Some Experiment Hyper-parameters.}
    \label{tab:hyper}
    
    \resizebox{0.9\linewidth}{!}{
    \begin{tabular}{lcccccc}
    \toprule
    Dataset & num\_neighbors & num\_epoch & patience & $mem\_dim$ & $emb\_dim$ & $time\_dim$ \\
    \midrule
    Wiki & 15 & 20 & 5 & 184 & 184 & 100\\
    Review  & 15 & 10 & 3& 184 & 184 & 100\\
    Coin & 10 & 5 & 3 & 100 & 100 & 100\\
    Comment & 10 & 3 & 2 & 100 & 100 & 100\\
    \bottomrule
    \end{tabular}}
    
\end{table}

\section{Additional Experimental Results}
\label{app:exps}

\subsection{Transductive and Inductive Experiments on Previously Small and Medium Datasets}

In addition to the large-scale TGB dataset, we also conduct experiments on some traditional small and medium datasets previously used in dynamic graph link prediction. We follow TGN \citep{rossi2020temporal} to evaluate different models under both transductive and inductive settings. The transductive setting deal with the future
links between previously observed nodes in the training stage, and the inductive setting predicts link existence between unseen nodes. We compare TNCN with the aforementioned baselines with the addition of FreeDyG \citep{tian2023freedyg}, which is also a competitive method in temporal graph link prediction. The overall performance of TNCN and different models are in Table \ref{tab:trans-ind}.

From Table~\ref{tab:trans-ind} we can find that TNCN achieves \textbf{6 new SOTA} out of 7 datasets \textbf{in transductive setting} and \textbf{3 in inductive setting}, furthur uncovering the strong performance of our model.

\begin{table}[htp] 
    \begin{center}
    
    \caption{Average Precision (AP) under Transductive and Inductive settings on small and medium dataset. The best is in \textbf{bold} font, and the second is \underline{underlined}. (``Trans'' and ``Ind'' are the abbreviation for transductive and inductive respectively.)}
    \label{tab:trans-ind}

    \resizebox{\textwidth}{!}{
    \begin{tabular}{c|cccccccccccccc} 
    \toprule 
    Setting & Method & Wikipedia & Reddit & Mooc & Lastfm & Enron & Social Evo. & UCI & Avg. Rank \\
    \midrule
    \multirow{15}{*}{Trans} & CAWN & 98.62$\pm$0.05 & 98.66$\pm$0.09 & 80.15$\pm$0.25 & 86.99$\pm$0.06 & 89.56$\pm$0.09 & 84.96$\pm$0.09 & 95.18$\pm$0.06 & 8.71 \\
    & JODIE & 96.15$\pm$0.36 & 97.20$\pm$0.05 & 80.23$\pm$2.44 & 70.85$\pm$2.13 & 84.77$\pm$0.30 & 89.89$\pm$0.55 & 89.43$\pm$1.09 & 11.57 \\
    & DyRep & 95.81$\pm$0.15 & 98.00$\pm$0.19 & 81.97$\pm$0.49 & 71.92$\pm$2.21 & 82.38$\pm$3.36 & 88.87$\pm$0.30 & 65.14$\pm$2.30 & 11.86 \\
    & TGAT & 96.94$\pm$0.06 & 98.52$\pm$0.02 & 85.84$\pm$0.15 & 73.42$\pm$0.21 & 71.12$\pm$0.97 & 93.16$\pm$0.17 & 79.63$\pm$0.70 & 10.43 \\
    & NAT & 98.68$\pm$0.04 & 99.10$\pm$0.09 & 86.54$\pm$0.02 & 88.56$\pm$0.02 & 92.42$\pm$0.09 & 94.43$\pm$1.67 & 94.37$\pm$0.21 & 6.29 \\
    & TCL & 96.47$\pm$0.16 & 97.53$\pm$0.02 & 82.38$\pm$0.24 & 67.27$\pm$2.16 & 79.70$\pm$0.71 & 93.13$\pm$0.16 & 89.57$\pm$1.63 & 11.29 \\
    & DyGFormer & 99.03$\pm$0.02 & 99.22$\pm$0.01 & 87.52$\pm$0.49 & 93.00$\pm$0.12 & 92.47$\pm$0.12 & 94.73$\pm$0.01 & 95.79$\pm$0.17 & 4.64 \\
    & FreeDyG & \underline{99.26$\pm$0.01} & \underline{99.48$\pm$0.01} & 89.61$\pm$0.19 & 92.15$\pm$0.16 & 92.51$\pm$0.05 & \underline{94.91$\pm$0.01} & 96.28$\pm$0.11 & 3.14 \\
    & TPNet & \textbf{99.32$\pm$0.03} & 99.27$\pm$0.01 & \underline{96.39$\pm$0.09} & \underline{94.50$\pm$0.08} & \underline{92.90$\pm$0.17} & 94.73$\pm$0.02 & \underline{97.35$\pm$0.04} & \underline{2.21} \\
    & EdgeBank & 90.37$\pm$0.00 & 94.86$\pm$0.00 & 57.97$\pm$0.00 & 79.29$\pm$0.00 & 83.53$\pm$0.00 & 74.95$\pm$0.00 & 76.20$\pm$0.00 & 12.43 \\
    & GraphMixer & 97.25$\pm$0.03 & 97.31$\pm$0.01 & 82.78$\pm$0.15 & 75.61$\pm$0.24 & 82.25$\pm$0.16 & 93.37$\pm$0.07 & 93.25$\pm$0.57 & 9.71 \\
    & CNE-N & 99.09$\pm$0.04 & 99.22$\pm$0.01 & 94.18$\pm$0.07 & 93.55$\pm$0.12 & 92.48$\pm$0.10 & 94.60$\pm$0.03 & 96.85$\pm$0.08 & 3.64 \\
    & TGN & 98.57$\pm$0.05 & 98.70$\pm$0.03 & 89.15$\pm$1.60 & 77.07$\pm$3.97 & 86.53$\pm$1.11 & 93.57$\pm$0.17 & 92.34$\pm$1.04 & 7.57 \\
    & TNCN & 99.03$\pm$0.02 & \textbf{99.79$\pm$0.02} & \textbf{96.69$\pm$0.04} & \textbf{98.65$\pm$0.06} & \textbf{97.08$\pm$0.14} & \textbf{99.95$\pm$0.04} & \textbf{97.44$\pm$0.08} & \textbf{1.50} \\
    \midrule
    \multirow{15}{*}{Ind} & CAWN & 98.24$\pm$0.03 & 98.19$\pm$0.03 & 81.42$\pm$0.24 & 89.42$\pm$0.07 & 86.35$\pm$0.51 & 79.94$\pm$0.18 & 92.73$\pm$0.06 & 7.71 \\
    & JODIE & 94.82$\pm$0.20 & 96.50$\pm$0.13 & 79.63$\pm$1.92 & 81.61$\pm$3.82 & 80.72$\pm$1.39 & 91.96$\pm$0.48 & 79.86$\pm$1.48 & 10.00 \\
    & DyRep & 92.43$\pm$0.37 & 96.09$\pm$0.11 & 81.07$\pm$0.44 & 83.02$\pm$1.48 & 74.55$\pm$3.95 & 90.04$\pm$0.47 & 57.48$\pm$1.87 & 11.29 \\
    & TGAT & 96.22$\pm$0.07 & 97.09$\pm$0.04 & 85.50$\pm$0.19 & 78.63$\pm$0.31 & 67.05$\pm$1.51 & 91.41$\pm$0.16 & 79.54$\pm$0.48 & 10.50 \\
    & NAT & 98.55$\pm$0.09 & 98.56$\pm$0.21 & 78.16$\pm$0.01 & 85.91$\pm$0.02 & \textbf{94.94$\pm$1.15} & \underline{95.16$\pm$0.66} & 92.58$\pm$1.86 & 5.71 \\
    & TCL & 96.22$\pm$0.17 & 94.09$\pm$0.07 & 80.60$\pm$0.22 & 73.53$\pm$1.66 & 76.14$\pm$0.79 & 91.55$\pm$0.09 & 87.36$\pm$2.03 & 10.93 \\
    & DyGFormer & 98.59$\pm$0.03 & 98.84$\pm$0.02 & 86.96$\pm$0.43 & 94.23$\pm$0.09 & 89.76$\pm$0.34 & 93.14$\pm$0.04 & 94.54$\pm$0.12 & 4.71 \\
    & FreeDyG & \textbf{98.97$\pm$0.01} & \underline{98.91$\pm$0.01} & 87.75$\pm$0.62 & 94.89$\pm$0.01 & 89.69$\pm$0.17 & 94.76$\pm$0.05 & 94.85$\pm$0.10 & 3.14 \\
    & TPNet & \underline{98.91$\pm$0.01} & 98.86$\pm$0.01 & \textbf{95.07$\pm$0.26} & \underline{95.36$\pm$0.11} & 90.34$\pm$0.28 & 93.24$\pm$0.07 & \textbf{95.74$\pm$0.05} & \textbf{2.43} \\
    & EdgeBank & / & / & / & / & / & / & / & / \\
    & GraphMixer & 96.65$\pm$0.02 & 95.26$\pm$0.02 & 81.41$\pm$0.21 & 82.11$\pm$0.42 & 75.88$\pm$0.48 & 91.86$\pm$0.06 & 91.19$\pm$0.42 & 9.43 \\
    & CNE-N & 98.37$\pm$0.03 & 98.78$\pm$0.01 & \underline{91.89$\pm$0.31} & 94.64$\pm$0.12 & 89.66$\pm$0.22 & 93.29$\pm$0.37 & \underline{95.03$\pm$0.16} & 4.00 \\
    & TGN & 97.83$\pm$0.04 & 97.50$\pm$0.07 & 89.04$\pm$1.17 & 81.45$\pm$4.29 & 77.94$\pm$1.02 & 90.77$\pm$0.86 & 88.12$\pm$2.05 & 8.57 \\
    & TNCN & 98.31$\pm$0.05 & \textbf{99.07$\pm$0.02} & 91.56$\pm$0.23 & \textbf{95.74$\pm$0.50} & \underline{92.04$\pm$0.22} & \textbf{95.51$\pm$0.07} & 94.57$\pm$0.17 & \underline{2.57} \\
    \bottomrule
    \end{tabular}}
    \end{center}
    
\end{table}


\subsection{Comparison with Some Classic Heuristic Methods}
We exhibits the result between TGN with some classic heuristics and TNCN under official setting on tgbl-wiki dataset in the Table~\ref{tab:heu}. Here heuristics consist of CN~\citep{doi:10.1126/science.286.5439.509}, RA~\citep{zhou2009predicting}, AA~\citep{adamic2003friends}, PPR~\citep{page1999pagerank, jeh2003scaling} and ELPH~\citep{chamberlain2022graph}. In these heuristic methods, the heuristic statistics are concatenated with TGN embedding to obtain final predictions. From the table we can see that these basic heuristics such as CN and RA do not bring performance improvement. However, some sophisticated heuristics like graph sketching in ELPH can enhance the backbone's capability. Nevertheless, using these heuristics cannot outperform a more generalized model like our TNCN, which comprehensively takes neighborhood nodes' representations into account.

\begin{table}[htp]
\begin{center}
\caption{Comparison between TGN with heuristics and TNCN on tgbl-wiki Dataset.}
\label{tab:heu}

\resizebox{0.85\textwidth}{!}{
\begin{tabular}{lcccc}
\toprule
Model & Val MRR & Test MRR & Training Time (s) & Inference Time (s) \\ 
\midrule
TGN & 0.569 & 0.528 & 10.33 & 98.74 \\ 
TGN-CN & 0.561 & 0.505 & 12.33 & 106.21 \\ 
TGN-RA & 0.563 & 0.511 & 16.51 & 115.04 \\ 
TGN-AA & 0.565 & 0.517 & 11.42 & 115.01 \\ 
TGN-PPR & 0.521 & 0.427 & 207.01 & 327.22 \\ 
TGN-ELPH & 0.715 & 0.681 & 240.92 & 1614.86 \\ 
TNCN & 0.742 & 0.724 & 21.45 & 250.49 \\ 
\bottomrule
\end{tabular}}
\end{center}
\end{table}

\subsection{Detailed Statistics of Parameter Analysis}
\label{app:params-app}

Here we provide some results on parameter analysis of our TNCN model.  Figure \ref{fig:params-fig} (a) illustrates the model performance under different \textbf{numbers of neighbors}. Incorporating more recent neighbors can retain richer information, while it may bring about some irrelevant noise. We can find that the test mrr increases as num\_neighbors rises from a small value, but it declines after a certain threshold. In Figure \ref{fig:params-fig} (b), we examine how the \textbf{embedding and memory dimension} affect the final performance. A similar trend is observed for this parameter. On the other hand, both the training and test time consistently increase with the parameter values becoming larger. This phenomenon is revealed in the Tables \ref{tab:para1} to \ref{tab:para4}. In Table \ref{tab:para1} and \ref{tab:para2}, we show the detailed MRR and time consumption of our TNCN with different numbers of neighbors over Wiki and Coin dataset. In Table \ref{tab:para3} and \ref{tab:para4} we show how the embedding and memory dimension influence the final performance. To strike a balance between the model capability and time consumption, we finally select the intermediate and appropriate values.

\begin{figure}[htp]

    \centering
    \subfigure[num\_neighbors]{
        \includegraphics[width=0.35\linewidth]{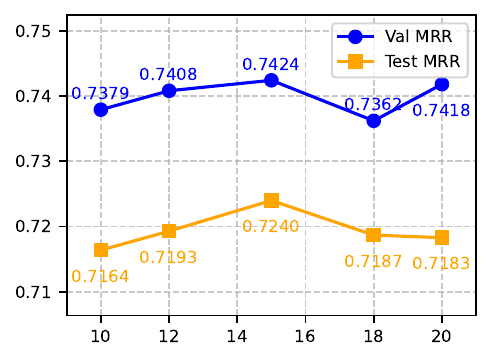}
    }
    \subfigure[emb\_dim \& mem\_dim]{
        \includegraphics[width=0.35\linewidth]{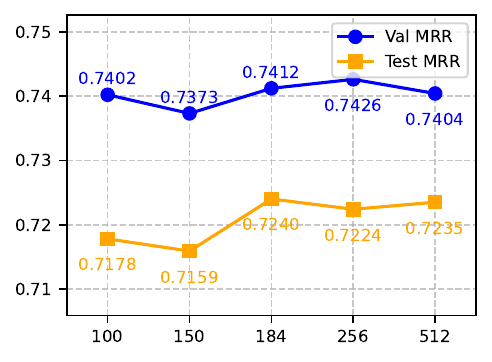}
    }
        
    \caption{Parameter Analysis on Wiki Dataset.}
    \label{fig:params-fig}
    
\end{figure}

\begin{table}[h]
    \centering
    \caption{Performance metrics for different numbers of neighbors on Wiki dataset.}
    \label{tab:para1}
    
     \resizebox{0.75\linewidth}{!}{
     \begin{tabular}{lccccc}
        \toprule
        num\_neighbors & 10     & 12     & 15     & 18     & 20     \\
        \midrule
        Val MRR            & 0.7379 & 0.7408 & 0.7412 & 0.7362 & 0.7418 \\
        Test MRR           & 0.7164 & 0.7193 & 0.7240 & 0.7187 & 0.7183 \\
        training time (s)  & 26.35  & 27.17  & 29.49  & 30.32  & 31.56  \\
        test time (s)      & 378.88 & 396.75 & 407.43 & 413.83 & 420.22 \\
        \bottomrule
    \end{tabular}
    }
\end{table}

\begin{table}[h]
    \centering
    \caption{Performance metrics for different numbers of neighbors on Coin dataset.}
    \label{tab:para2}
    
    \resizebox{0.85\linewidth}{!}{
    \begin{tabular}{lccccc}
        \toprule
        num\_neighbors & 5       & 8       & 10      & 12      & 15      \\
        \midrule
        Val MRR            & 0.7492  & 0.7378  & 0.7430  & 0.7450  & 0.7406  \\
        Test MRR           & 0.7687  & 0.7619  & 0.7701  & 0.7662  & 0.7601  \\
        training time (s)  & 5936.19 & 6129.33 & 6406.00 & 6911.00 & 7529.01 \\
        test time (s)      & 56605.45 & 57117.99 & 57292.00 & 57745.00 & 57925.00 \\
        \bottomrule
    \end{tabular}
    }
\end{table}

\begin{table}[h]
    \centering
    \caption{Performance metrics for different embedding and memory dimensions on Wiki dataset.}
    \label{tab:para3}
    
    \resizebox{0.8\linewidth}{!}{
    \begin{tabular}{lccccc}
        \toprule
        emb\_dim\&mem\_dim & 100     & 150     & 184     & 256     & 512     \\
        \midrule
        Val MRR                & 0.7402  & 0.7373  & 0.7412  & 0.7426  & 0.7404  \\
        Test MRR               & 0.7178  & 0.7159  & 0.7240  & 0.7224  & 0.7235  \\
        training time (s)      & 22.34   & 25.47   & 29.49   & 32.81   & 34.95   \\
        test time (s)          & 378.28  & 399.77  & 407.43  & 420.45  & 459.86  \\
        \bottomrule
    \end{tabular}
    }
\end{table}

\begin{table}[h]
    \centering
    \caption{Performance metrics for different embedding and memory dimensions on Coin dataset.}
    \label{tab:para4}
    
    \resizebox{0.8\linewidth}{!}{
    \begin{tabular}{lccccc}
        \toprule
        emb\_dim\&mem\_dim & 30      & 50      & 100     & 150     & 184     \\
        \midrule
        Val MRR                & 0.7387  & 0.7405  & 0.7430  & 0.7436  & 0.7518  \\
        Test MRR               & 0.7591  & 0.7606  & 0.7701  & 0.7646  & 0.7699  \\
        training time (s)      & 6228    & 6340    & 6406    & 6617    & 6721    \\
        test time (s)          & 56694   & 57179   & 57292   & 58103   & 59411   \\
        \bottomrule
    \end{tabular}
    }
\end{table}

\subsection{Comparison between TNCN and traditional NCN}
\label{app:ncn}
Here we provide a detailed comparison between our TNCN model and the traditional NCN method in Table~\ref{tab:cmp}. 
\begin{table}[htp]

\caption{The comparison of TNCN and NCN.}
\label{tab:cmp}
\centering

\resizebox{\linewidth}{!}{
\begin{tabular}{|l|l|l|l|l|}
\hline
     & temporal scenario & backbone        & arbitrary CN hops & batch-wise CN extraction        \\ \hline
NCN  &  \XSolidBrush     & traditional GNN & \XSolidBrush    & \XSolidBrush \\ \hline
TNCN & \CheckmarkBold     & memory-based    & \CheckmarkBold         & \CheckmarkBold    \\ \hline
\end{tabular}
}

\end{table}

\section{Detailed Time and Memory Consumption Statistics}

\subsection{Time Consumption}
\label{app:time}
Table \ref{tab:efficiency-4090} exhibits the detailed time consumption on TGB datasets with different models. We can observe that TNCN maintains similar time consumption to memory-based networks while achieving striking speedup compared with graph-based models. 


    


\begin{table}[htp]

\begin{center}
    \caption{Time Consumption of different methods on TGB Datasets. All these experiments are conducted with NVIDIA GeForce RTX 4090.}
    
    \label{tab:efficiency-4090}
\resizebox{\textwidth}{!}{
    \begin{tabular}{lccccc}
\toprule
Model(tr/val/test)(s) & Wiki            & Review            & Coin             & Comment             & Flight              \\ 
\midrule
TGN                   & 12/41/43      & 988/1389/1411    & 2578/5219/5408   & 6599/9984/10311    & 2990/19834/20627                \\ 
DyGFormer             & 85/6268/6317 & 3891/25831/26911  &   12593/50893/51276              &  OOM              & OOM \\ 
GraphMixer   & 27/2385/2425 & 983/14188/12764  & 5917/71699/72103  &       OOM   &            OOM   \\ 
CNE-N & 19/422/424 & 540/4749/4931 & 2308/18226/19317 & OOM & OOM \\
DyRep                 & 22/27/29        & 1514/877/916    &    6305/3821/3746   &   13701/6419/6587      &   9352/12188/13112     \\ 
TGAT                  & 81/6359/6407   & 2564/33437/34814  &  25514/73372/73821    &  OOM           &      OOM        \\ 
CAWN                  & 173/15690/15881 & 5692/79566/80717 &   OOM       &      OOM       &    OOM       \\ 
TCL                   & 31/680/682    & 948/9093/9962    &   4383/52973/54016      &     OOM    &    OOM    \\ 
TNCN          & 33/565/566      & 971/2626/2701  & 5926/54765/55177 & 9828/59439/60163    & 8209/50533/51064 \\ 
\bottomrule
\end{tabular}}  
\end{center}   

\end{table}

\noindent\textbf{Comparison between TNCN and NAT}
Here we also show some experimental results in Table~\ref{tab:nat} for the comparison between TNCN and NAT model. The hardware we use is NVIDIA GeForce RTX 2080 as NAT's code isn't compatible with higher version. Note that NAT model exposes a backward as its instability, accomplishing about only 1/3 experiments when we test it.

\begin{table}[htp]
\begin{center}
\caption{Comparison of Time Consumption between TNCN and NAT.}
\label{tab:nat}

\resizebox{0.6\textwidth}{!}{
\begin{tabular}{lcccc}
\toprule
Dataset & Model & Train (s) & Val (s) & Test (s) \\ 
\midrule
tgbl-wiki & TNCN & 21.45 & 250.49 & 251.52 \\ 
& NAT & 74.92 & 298.6 & 298.41 \\ 
\midrule
tgbl-review & TNCN & 1649 & 4788 & 4695 \\ 
& NAT & 422 & 7516 & 7461 \\ 
\midrule
tgbl-coin & TNCN & 4920 & 28716 & 28805 \\ 
& NAT & 1896 & 30398 & 30176 \\ 
\bottomrule
\end{tabular}}
\end{center}
\end{table}

\subsection{Memory Consumption}
\label{app:memory}
In Table \ref{tab:memory} we have shown the memory cost of different methods on TGB datasets. On most of them (except Review) we can observe that TNCN takes similar GPU memory consumption to memory-based networks while occupying significant less memory than graph-based models. 

\begin{table}[htp]

\begin{center}
    \caption{Memory Consumption of different methods on TGB Datasets. All these experiments are conducted with NVIDIA GeForce RTX 4090. (``-'' stands for out of memory)}
    
    \label{tab:memory}
\resizebox{0.8\textwidth}{!}{
    \begin{tabular}{lccccc}
\toprule
Inference Stage(MB) & Wiki            & Review            & Coin             & Comment             & Flight              \\ 
\midrule
TGN                   & 1050    & 6416    & 11496   & 13838    & 1182                \\ 
DyGFormer     & 3174 & 5512  &   17428              &       -              &            -         \\ 
GraphMixer   & 4914 & 5388  &   17360    &       -              &            -         \\ 
CNE-N & 1154 & 4628 & 16760 & - & - \\
DyRep                 & 998        & 6416    &    9300   &   10418      &   1126    \\ 
TGAT     & 5600   & 6588  &    18526         &      -               &         -            \\ 
CAWN        & 7656 & 10898 &     -    &      -               &      -               \\ 
TCL          & 2204    & 5276    &  17292             &       -              &           -          \\ 
TNCN          & 1420     & 9586  & 11990 & 13842    & 1188 \\ 
\bottomrule
    \end{tabular}}  
\end{center}   

\end{table}

\subsection{Overall Comparison between Performance and Time/Memory Consumption}
Here we provide a scatter plot to compare different methods between the final performance and the time/memory consumption. Figure \ref{fig:tm-cmp} (a) shows the result on Wiki, and Figure \ref{fig:tm-cmp} (b) on Coin. From these scatter plots, we can obtain a more intuitive perception that TNCN can better achieve a balance between the final performance and time/memory consumption.
\begin{figure}[htp]

    \centering
    \subfigure[Wiki]{
        \includegraphics[width=0.48\linewidth]{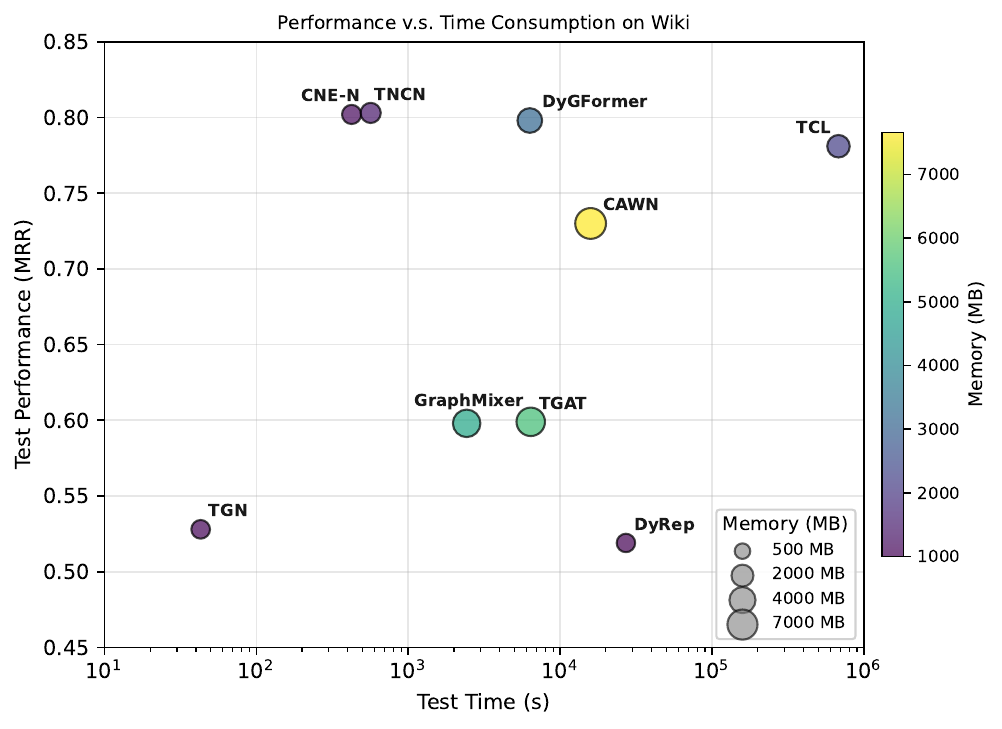}
    }
    \subfigure[Coin]{
        \includegraphics[width=0.48\linewidth]{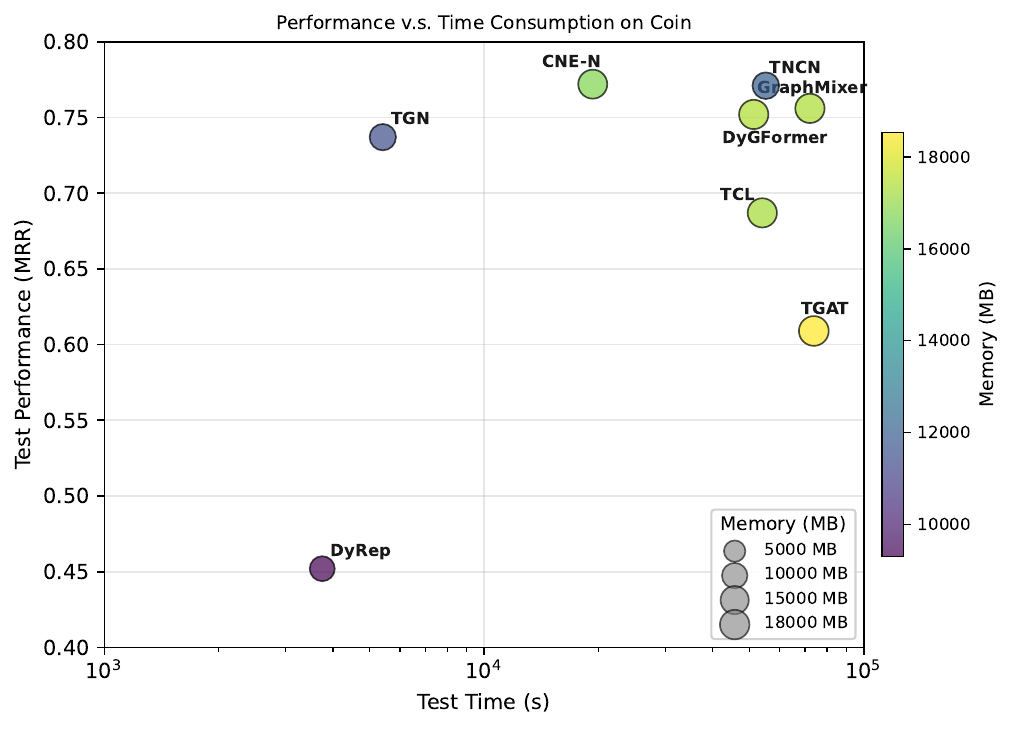}
    }
        
    \caption{Comparison between Performance and Time/Memory Consumption.}
    \label{fig:tm-cmp}
    
\end{figure}

\section{Pseudocode of TNCN Pipeline and CN Extraction}
\label{app:pseudocode}
Algorithm~\ref{alg:code} shows the pseudocode about the pipeline of our TNCN model, and Algorithm~\ref{alg:cn-extraction} for the procedure of the CN extraction operation.

\begin{algorithm}[htp]
\caption{Pipeline of TNCN}
\label{alg:code}
\begin{algorithmic}[1]
\For {positive batch data $(\text{posu}, \text{posv}, t)$}
    \State $\text{neg\_batch} \leftarrow \text{negative sampling}$
    \For {batch\_data in \{\text{pos\_data}, \text{neg\_data}\}} 
        \State $\text{mem}, \text{hist\_events} \leftarrow \text{memory\_module(batch\_data)}$;\\ \Comment{Get node memory and historical interactions}
        \State $\text{emb} \leftarrow \text{transform(mem, hist\_events)}$;\\ \Comment{Get the node embedding}
        \State $\text{CN\_mat} \leftarrow \text{CN\_extractor(hist\_events)}$;\\ \Comment{Obtain the CNs for given node pair in a batch}
        \State $\text{NCN\_emb} \leftarrow \text{AGG(emb, CN\_mat)}$;\\ \Comment{Aggregate the embeddings of CNs}
        \State $p \leftarrow \text{Pred(emb\_u, emb\_v, NCN\_emb)}$;\\ \Comment{Calculate the probability of future links}
    \EndFor    
    \State $\text{mem} \leftarrow \text{memory\_update(pos\_data)}$;\\ \Comment{Update the memory module}
\EndFor
\end{algorithmic}
\end{algorithm}

\begin{algorithm}[htp]
\caption{Procedure of CN Extraction}
\label{alg:cn-extraction}
\begin{algorithmic}[1]
\State \textbf{Input:} Temporal adjacency matrix of node $u$ and $v$, denoted as $A\in \mathbb{R}^{N \times N}$ and $B\in \mathbb{R}^{N \times N}$. $A$ and $B$ are the sub-matrix from the whole adjacency matrix of the temporal graph, containing $u$, $v$ and their individual neighbors. $N$ is the total number of their temporal neighbors and themselves.
\State \textbf{Output:} $k_{th}$ hop Common Neighbors of $u$ and $v$, $CN_k\in \mathbb{R}^{N}$.\\
\State \textbf{Stage 1:} Generate up to $k$-hop neighbors. \\
\quad \quad $A \leftarrow A + I$, $B \leftarrow B + I$. \Comment{add self-loop}\\
\quad \quad Compute $A^{k-1}$ and $A^k$, $B^{k-1}$ and $B^k$. \Comment{obtain up to $k-1$ / $k$ hop adjacency matrix}\\
\quad \quad $N_{0\sim k-1}(u) = A^{k-1}[id(u)]$,  $N_{0\sim k}(u) = A^{k}[id(u)]$. \Comment{Similar to $v$. $N(u)\in \mathbb{R}^{N}$}\\
\State \textbf{Stage 2:} Get $0\sim k-1$ and $0\sim k$ hop common neighbors. \\
\quad \quad $CN_{0\sim k-1} = N_{0\sim k-1}(u) \odot N_{0\sim k-1}(v)$, \Comment{$CN\in \mathbb{R}^{N}$}\\ \quad \quad $CN_{0\sim k} = N_{0\sim k}(u) \odot N_{0\sim k}(v)$. \Comment{perform sparse matrix hadamard product}\\
\State \textbf{Stage 3:} Get exact $k$-hop common neighbors. \\
\quad \quad $CN_{k} = CN_{0\sim k} - CN_{0\sim k-1}$. \Comment{perform sparse matrix subtraction}\\
\State Finally we get the exact $k$-hop common neighbors between node $u$ and $v$.
\end{algorithmic}
\end{algorithm}

\section{Special Cases Analysis of Common Neighbor Extraction}
\label{app:cn}


Here are some special cases while calculating $(1,2)$, $(2,1)$ and $(2,2)$ hop CNs. Under these situations, utilizing $A^k[id(u)]$ naively in step (2) will lead to walk-based neighbors, i.e. $\exists v, \exists i\neq j, w_i=w_j, \ s.t. \  (u, w_1, w_2, \cdots, w_{k-1}, v) \ exists.$ To obtain a clear version of arbitrary path-based $(i,j)$-hop CNs, we should avoid the repetition of neighbors. We take $(1,2)$ as an example to analyse the detailed method to eliminate repetition. Cases like $(2,1)$ and $(2,2)$ hop can be similarly solved. 

Assume that node $x$ is a $(1,2)$-hop CN of pair $(u,v)$, thus we know $\exists w, \ s.t.$ $(u,x)$ and $(v,w,x)$ exist. There are two variants that render $x$ to be a walk-based CN instead of a path-based one that we exactly require. 

(a) $x=v$. When $x=v$, the local graph has the topology shown in igure~\ref{fig:special} a. This situation should satisfy two conditions: $w$ is a neighbor of $v$ and there are historical interactions between $u$ and $v$. Denote the frequency between $(u,v)$ before time $t$ as $q^t(u,v) = |\{(u,v,t')|t'<t\} \cup \{(v,u,t')|t'<t\}|$. So the naively computed $CN^t_{(1,2)}(u,v)[id(x)]$ value need to be subtracted by $[\sum\limits_{w_i}q^t(w_i,v)] * q^t(u,v)$, i.e. the total interaction frequency of $v$ before time $t$ multiplied by the frequency between $(u,v)$.

(b) $w=u$. The structure is exhibited in Figure~\ref{fig:special} b. Here $(u,v)$ has historical edges and $x$ is a $1$-hop neighbor of $u$. The additive substraction value is $[\sum\limits_{x}q^t(x,u)^2] * q^t(u,v)$.

(c) Both (1a) and (1b) are satisfied. The ground truth is as Figure~\ref{fig:special} c. We just need to add back the overlap value that have been diminished once more.

\begin{figure}[tp]
    \centering
    \subfigure[$x=v$]{
        \includegraphics[width=0.20\textwidth]{fig/1-2-xv.pdf}
    }
    \subfigure[$w=u$]{
        \includegraphics[width=0.20\textwidth]{fig/1-2-wu.pdf}
    }
    \subfigure[$x=v$ \& $w=u$]{
        \includegraphics[width=0.20\textwidth]{fig/1-2-both.pdf}
    }
    \caption{Here shows the special cases related to $(1,2)$-hop CNs computation. Note that the graph is undirected, while the directed arrows implies the path direction used to determine the corresponding hop numbers.}
    \label{fig:special}
\end{figure}

Note that the procedure above can only deal with CNs of \textbf{no more than $(2,2)$-hop} perfectly. For higher-order $(i,j)$-hop CN extraction, please refer to \citet{perepechko2009number} for more details and complicated analysis.

\section{Case Study}
\subsection{Two Examples from TGB for Better Understanding TNCN's Effectivity}
In Figure~\ref{fig:case}, we show two case studies from TGB to give a better understanding of the effectivity of our TNCN.

\begin{figure}[tp]
    \centering
    \subfigure[tgbl-wiki]{
        \includegraphics[width=0.35\textwidth]{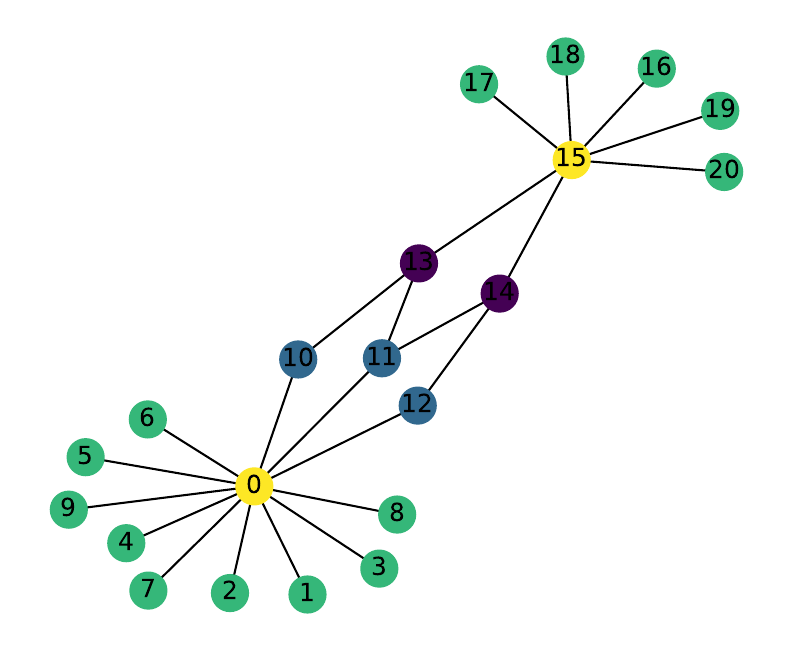}
    }
    \subfigure[tgbl-coin]{
        \includegraphics[width=0.35\textwidth]{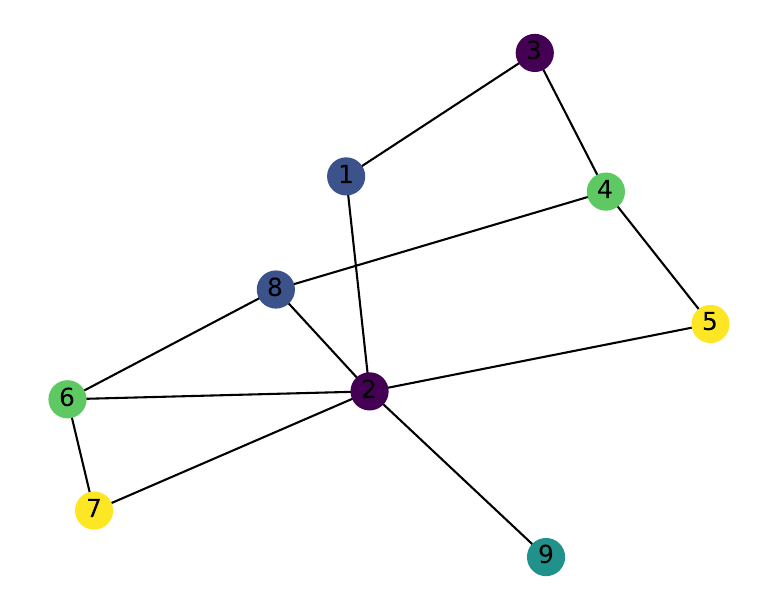}
    }
    \caption{Two case studies from TGB.}
    \label{fig:case}
\end{figure}

Figure (a) shows a case from tgbl-wiki, which is a bipartite graph. The yellow nodes 0 and 15 are a $(u,v)$ pair. If we use a node-wise method to predict the future link of $(0,15)$, we can find that node 15 has just 7 neighbors while node 0 has 12. So their properties may be different, thus having less chance to have an interaction. However TNCN can observe that the blue nodes are their $(1,2)$-hop CNs and purple nodes are the $(2,1)$-hop, and it will give a high probability over the existence of the future link.

Figure (b) shows another case from tgbl-coin. Here we need to predict the link of $(1,8)$. Here TNCN can find that these two nodes have multiple variants of common neighbors. Node 3 is their $(1,2)$-hop CN, node 4 and 6 are the $(2,1)$-hop, and node 2 is both their $(1,1)$ and $(2,1)$ hop. The $(2,2)$-hop CNs are node 5 and 7, while node 9 being a special $(2,2)$-hop CN that owns a shared 1-hop node 2. With the abundant CN information, TNCN will be more likely to predict it as a positive future edge.

\subsection{Some Statistics about Common Neighbors on Comment Dataset}
Here we show some statistics about the structural information of common neighbors on Comment dataset, which may account for the lower performance of TNCN. We have collected four variants of statistics including (1) ``s-nei/d-nei'': the number of recent neighbors (without repetition) of the positive src/dst nodes, (2) ``cn'': the number of common neighbors of the positive node pairs, and (3) ``prev'': their previous interaction times recently. All the values are the average from the test set. The statistics are shown in the following Table~\ref{tab:CN-Comment}.

\begin{table}[htp]
    \centering
    \caption{Statistics about CNs on Comment Dataset. (``s'' stands for the ``surprise'' event, ``ns'' for not-surprise; ``p'' for the case where the ground-truth event ranks 1st in our prediction, ``np'' for not.)}
    \label{tab:CN-Comment}
    \begin{tabular}{ccccc}
\toprule
 & s-nei & d-nei & cn & prev \\
\midrule
s-p & 8.43 & 9.03 & 0.87 & 0.39 \\
s-np & 8.89 & 8.78 & 0.07 & 0.02 \\
ns-p & 7.83 & 8.05 & 2.64 & 0.69 \\
ns-np & 8.76 & 8.88 & 0.17 & 0.07 \\
\bottomrule
\end{tabular}
\end{table}

From the table we can find that, although src/dst nodes from both surprise and not-surprise edges have similar number of neighbors, their common neighbors and the previous contact frequecy are distinct far away. Nodes from surprising edges always share less CNs and interaction times, making CN features less effective in the prediction. From another aspect, the successfully predicted events typically have their nodes with more CNs and historical interactions, which indicates that our TNCN model may be misled by the node pair with more CNs in the prediction.

\section{Proofs}
\label{app:proof}

In this section, we give proofs on theorems \ref{the:encoding}, \ref{the:time} and \ref{the:cn}.


\begin{theorem}
(Theorem \ref{the:encoding}. Ability of encoding $k$-hop event). Given a $k$-hop event $\{(u_i, u_{i+1}, t_{u_{i}, u_{i+1}}) \mid i \in \{0,\ldots,k-1\}, k\geq1\}$. If the node embedding of $u_0$ at time $t_{u_{0}, u_{1}}$, can be formally derived by the encoding function $\textit{Enc}(\{(u_i, u_{i+1}, t_{u_{i}, u_{i+1}}) \mid i \in \{0,\ldots,k-1\}, k\geq1\})$, then the learning method is considered capable of encoding the $k$-hop event. The following outline the encoding capabilities of different learning schemes:
\begin{itemize}
    \item Memory-based approach can encode any $k$-hop events with $k = 1$.
    \item Memory-based approach can encode any monotone $k$-hop events with arbitrary $k$.
    \item $k$-hop-subgraph-based approach can encode any $k'$-hop events with $k' \leq k$
\end{itemize}
\end{theorem}
\begin{proof}
    In the following analysis, we establish the encoding efficacy of the memory-based approach. Consider a $k$-hop event with the simplifying assumption that $k=1$, which reduces the event to the tuple ${(u_0, u_1, t_{u_{0}, u_{1}})}$. By adhering to the predefined schematics of the memory-based methodology, the memory state $Mem(u_0, t_{u_{0}, u_{1}})$ is updated via the function $f_{\textit{mem}}$ such that $Mem(u_0, t_{u_{0}, u_{1}}) = \\ f_{\textit{mem}}(Mem(u_0, t’), e^{t_{u_{0}, u_{1}}}_{u_0, u_1}, t_{u_{0}, u_{1}} - t’)$. Let us denote the encoding function as $Enc = f_{\textit{mem}}(Mem(u_0, t’), \ldots)$. It is our intention to demonstrate that this memory-based framework is capable of encoding any $k$-hop event for $k = 1$. 
    
   We consider the encoding of an arbitrary monotonically increasing $k$-hop temporal event sequence within a memory-based approach. The induction principle is applied to demonstrate the capability of this approach. For the base case, $k=1$, the encoding has been shown to be feasible. Now, assume the proposition holds for a $k’$-hop event; that is, any $k’$-hop temporal sequence of monotonically increasing events can be encoded using a memory-based approach. This assumption implies that there exists an embedding function such that
   \begin{equation}
   Emb(u_0, t_{u_{0}, u_{1}}) = \textit{Enc}({(u_i, u_{i+1}, t_{u_{i}, u_{i+1}}) \mid i \in {0,\ldots,k’-1}}),
   \end{equation}
   for all event sequences with $k'$ hops, where $k’ \geq 1$. 
   Given an arbitrary $k’+1$-hop event, which can be partitioned into an initial event $(u_0, u_1, t_{u_{0}, u_{1}})$ and a subsequent $k’$-hop sequence. The existence of an encoding function for the $k’$-hop sequence assures that
       \begin{equation}
        \begin{aligned}
            Emb(u_0, t_{u_{0}, u_{1}}) &= f_{\textit{emb}}(Mem(u_{0}, t_{u_{1}, u_{2}})) \\
            &= f_{\textit{emb}}(f_{\textit{mem}}(Mem(u_{0}, t_{u_{1}, u_{2}}), \\ & \quad Mem(u_{1}, t_{u_{1}, u_{2}}), e_{u_{0}, u_{1}}^{t_{u_{0}, u_{1}}}, t_{u_{0}, u_{1}} - t_{u_{1}, u_{2}}), \\
          Mem(u_{1}, t_{u_{1}, u_{2}}) & = f_{\textit{emb}}^{-1}(Emb(u_1, t_{u_{1}, u_{2}})) \\
          & = f_{\textit{emb}}^{-1}(\textit{Enc}(\{(u_i, u_{i+1}, t_{u_{i}, u_{i+1}}) \\ & \mid i \in \{1,\ldots,k'\}, k'\geq1\})
        \end{aligned}
    \end{equation}
Subsequently, it is demonstrated that $Emb(u_0, t_{u_{0}, u_{1}})$ provides an encoding for both the initial event and the $k’$-hop sequence, thereby affirming its efficacy in encoding the entire $k’+1$-hop event. This concludes the inductive step and substantiates the inductive argument.

We consider a $k$-hop-subgraph-based approach for our analysis. It is evident that a $k$-hop subgraph encompasses any $k’$-hop events, where $k’ \leq k$. Furthermore, the aggregation methodology assimilates all nodes contained within the subgraph. Collectively, these observations substantiate the theorem in question.

\end{proof}

\begin{theorem}
    (Theorem \ref{the:time}. Learning method time complexity). Denote the time complexity of a learning method as a function of the total number of events processed during training. For a given graph $\gG$ with the number of nodes designated as $|\gN|$ and the number of edges as $|\gE|$, the following assertions hold:
    \begin{itemize}
        \item For the memory-based approach, the time complexity is $\Theta\left(|\gE|\right)$.
        \item For $k$-hop-subgraph-based with $k=1$, the lower-bound time complexity is $\Omega\left(\frac{|\gE|^2}{|\gN|}\right)$, and the upper-bound time complexity is $\gO\left(\frac{|\gE|^2}{|\gN|} + |\gE||\gN|\right)$
        \item For $k$-hop-subgraph-based with $k=2$, the upper-bound time complexity is $\gO\left((\frac{|\gE|^2}{|\gN|} + |\gE||\gN|)^{\frac{3}{2}}\right)$ 
    \end{itemize}
\end{theorem}

\begin{proof}
In the proposed theorem, the time complexity is denoted as the aggregate quantity of events processed throughout the training phase. The objective herein is to ascertain the precise count of such utilized events.

In the context of the memory-based methodology, it is evident that each event is utilized a singular time. Consequently, the cumulative number of events is expressed as $|\mathcal{E}|$, which infers that the time complexity adheres to the order of $\Theta\left(|\mathcal{E}|\right)$.

In the context of $k$-hop-subgraph-based algorithms wherein $k=1$, an event $(u,v,t)$ is exploited once for every incident event within the neighborhood of vertices $u$ or $v$. Without loss of generality, we focus on all events within the $1$-hop-subgraph of vertex $u$. The aggregate count of events processed is given by $\sum_{i=1}^{d(u)} i = \Theta\left( d(u)^2 \right)$, where $d(u)$ denotes the degree of vertex $u$. Consequently, the computational complexity is fundamentally proportional to $\sum_{u \in \gN} d(u)^2$. Drawing on the results of \citet{DECAEN1998245}, the lower bound on the time complexity is established as $\Omega\left(\frac{|\gE|^2}{|\gN|}\right)$, whereas the upper bound is determined as $\gO\left(\frac{|\gE|^2}{|\gN|} + |\gE||\gN|\right)$

In the context of $k$-hop-subgraph-based algorithms wherein $k=2$, we adopt similar strategy where each event$(u,v,t)$ will only be utilized once another event within the subgraph of $u$ or $v$ is firstly considered. The total number of events can be formulated as $\sum_{u \in \gN} d(u) \sum_{v \in \gN_u} \sum_{w \in \gN_v} d(w)$. Replacing $d(u)$ as $X_{i}$, $\sum_{v \in \gN_u} \sum_{w \in \gN_v}$ as $Y_{i}$, we reformulated is as $\sum_{i \in |\gN|} X_i Y_i$, satisfying $\sum_{i \in |\gN|} X_i^2 = \sum_{u \in \gN} d(u)^2$ and $\sum_{i \in |\gN|} Y_i^2 = \sum_{u \in \gN} d(u)^4$. Following Cauchy inequality and conclusions of $\sum_{u \in \gN} d(u)^2$, we got the the upper-bound time complexity is $\gO\left((\frac{|\gE|^2}{|\gN|} + |\gE||\gN|)^{\frac{3}{2}}\right)$ 
\end{proof}


\begin{theorem}
(Theorem \ref{the:cn}. Expressivity of TNCN)

\begin{enumerate}
    \item TNCN is strictly more expressive than CN, RA, and AA.
    \item TNCN is strictly more expressive than Jodie with the same dimension of time encoding, DyRep with the same aggregation function, TGAT with the same attention layers and neighbors, and TGN under identical condition for all module choices.
\end{enumerate}
\end{theorem}

\begin{proof}
(1) We first give definitions of these structural features under temporal settings. Given two nodes $u$ and $v$, the structural features before time $t$ are defined as follows:

\begin{equation}
\label{eq:cn}
    \begin{aligned}
        CN(u, v, t) & = \sum_{w \in N_{1}^{t}(u) \cap N_{1}^{t}(v)} 1, \\ 
        RA(u, v, t) & = \sum_{w \in N_{1}^{t}(u) \cap N_{1}^{t}(v)} \frac{1}{d(w)}, \\
        AA(u, v, t) & = \sum_{w \in N_{1}^{t}(u) \cap N_{1}^{t}(v)} \frac{1}{\log d(w)}
    \end{aligned}
\end{equation}
Given a node $u$, the degree of node $u$ is the number of events $e$ with an endpoint at node $u$. Without loss of generality (W.L.O.G.), we consider node $u$ as the source node, and the events are $\{ (u, v_{i}, t_{i}) \mid i \in \{0, \ldots, k-1\}, k \geq 1\}$. Each time a new event is given, the embedding of node $u$ is updated by
\begin{equation}
    mem_{u}^{t_{i}} = \textit{upd}_{src} (mem_{u}^{t_{i-1}}, \textit{msgfunc}_{src}(e_{u, v_{i}}^{t_{i}})).
\end{equation}

With the MPNN universal approximation theorem, $\textit{msgfunc}$ can be a constant function, and $\textit{upd}_{src}$ can be an addition function. Thus,
\begin{equation}
    mem_{u}^{t_{k-1}} = d(u).
\end{equation}
Then the embedding can learn arbitrary functions of node degrees, i.e.,
\begin{equation}
    emb_{u}^{t} = f(d(u)).
\end{equation}
Thus, the neural common neighbor \\ $TNCN_{1}(u, v) = \mathop{\oplus}_{w \in N_{1}^{t}(u) \cap N_{1}^{t}(v)} emb_{w}^{t}$ can express Equation \ref{eq:cn}.

Extending to situations where the common neighbor node has some features we want to learn, the traditional CN, RA, and AA cannot accommodate this. However, our TNCN can express these features, demonstrating that TNCN is strictly more expressive than CN, RA, and AA.

(2) We first prove that TNCN is strictly more expressive than TGN. As TNCN possesses the same memory-based framework as TGN to derive the $emb_u^t$, TNCN at least has the same expressivity as TGN with identical module choices. Additionally, from part (1) of this theorem we know that TNCN can capture certain heuristics such as CN where TGN fails. Hence TNCN is strictly more expressive than TGN. 

Since Jodie, DyRep and TGAT are specific instances of TGN with minor changes, we next shift our object to prove that TGN has no less expressivity than the three methods.

For TGAT, it retains only the Graph Attention mechanism for embedding functions, removing all memory and message modules present in TGN. Consequently, TGAT lacks the capacity to incorporate temporal memory updates and message passing, making it less expressive than TGN. 

Jodie differs from TGN in the memory updater and embedding function. Specifically, it utilizes an RNN rather than a GRU, although GRU is a general variant of RNN. Additionally, Jodie employs the time embedding as $emb_u^t = (1 + \Delta t \cdot w)\otimes mem_u^{t^-}$, where $w$ are learnable parameters. However, TGN's Graph Attention uses not only the memory $mem_u^{t^-}$ but also the historical edge features $e_{u,v}^{t'}$, and learns the coefficient via multi-head attention, thus achieving a higher expressivity.

DyRep also modifies the memory updater to  an RNN. Meanwhile, it changes the embedding function to \textit{Identity} and message function to \textit{Graph Attention}. Although it just swaps these two modules, the expressivity is reduced in the temporal setting. When the model attempts to make its prediction, it first generates the embeddings with the embedding function. So naturally the identity function processes less information in the $emb_u^t$ than graph attention. After obtaining the prediction result, the new messages will be calculated by the message function and fed into memory updater, where the $emb_u^t$ can not get access. Hence TGN is more expressive than DyRep, especially when the recent information can make large benefits.

In conclusion, TNCN is strictly more expressive than Jodie, DyRep, TGAT and TGN under the same condition.

\end{proof}


\end{document}

%% file: math_commands.tex
\usepackage{amsmath,amsfonts,bm}

\def\eqref#1{equation~\ref{#1}}

\def\1{\bm{1}}

\DeclareMathAlphabet{\mathsfit}{\encodingdefault}{\sfdefault}{m}{sl}
\SetMathAlphabet{\mathsfit}{bold}{\encodingdefault}{\sfdefault}{bx}{n}

\def\gE{{\mathcal{E}}}

\def\gG{{\mathcal{G}}}

\def\gN{{\mathcal{N}}}
\def\gO{{\mathcal{O}}}

